\documentclass{article}

\usepackage{microtype}
\usepackage{graphicx}
\usepackage{booktabs} 

\usepackage[accepted]{icml2021}

\usepackage{amsmath,amsthm,verbatim,amssymb,amsfonts,amscd, graphicx}
\usepackage{graphics}
\usepackage[colorlinks,linkcolor=blue,citecolor=blue]{hyperref}

\usepackage[bottom]{footmisc}
\usepackage{caption}
\usepackage{subcaption}
\usepackage{helvet}  
\usepackage{courier}  
\usepackage{url}  
\usepackage{graphicx}  
\usepackage{multirow}
\usepackage{amsthm}
\usepackage{color}
\usepackage{MnSymbol}
\usepackage{makecell}
\usepackage{arydshln}
\usepackage{amsmath}
\usepackage{xcolor}
\usepackage{caption} 
\usepackage{natbib}
\usepackage{textcomp}
\usepackage{wrapfig}
\usepackage{algorithm}
\usepackage{algorithmic}
\usepackage{bm}
\usepackage{mathtools}
\def\multiset#1#2{\ensuremath{\left(\kern-.3em\left(\genfrac{}{}{0pt}{}{#1}{#2}\right)\kern-.3em\right)}}

\theoremstyle{plain}
\newtheorem{theorem}{Theorem}
\newtheorem{corollary}{Corollary}
\newtheorem{lemma}{Lemma}
\newtheorem*{remark}{Remark}

\theoremstyle{definition}

\newtheorem*{theorem*}{Theorem}
\newtheorem*{corollary*}{Corollary}

\icmltitlerunning{Noise and Fluctuation of Finite Learning Rate Stochastic Gradient Descent}

\begin{document}

\twocolumn[
\icmltitle{Noise and Fluctuation of Finite Learning Rate Stochastic Gradient Descent}


\icmlsetsymbol{equal}{*}

\begin{icmlauthorlist}
\icmlauthor{Kangqiao Liu}{equal,to}
\icmlauthor{Liu Ziyin}{equal,to}
\icmlauthor{Masahito Ueda}{to,cems,ipi}
\end{icmlauthorlist}

\icmlaffiliation{to}{Department of Physics, the University of Tokyo, Japan}
\icmlaffiliation{cems}{RIKEN CEMS, Japan}
\icmlaffiliation{ipi}{Institute for Physics of Intelligence, the University of Tokyo, Japan}

\icmlcorrespondingauthor{Kangqiao Liu}{kqliu@cat.phys.s.u-tokyo.ac.jp}
\icmlcorrespondingauthor{Liu Ziyin}{zliu@cat.phys.s.u-tokyo.ac.jp}

\icmlkeywords{Machine Learning, ICML}

\vskip 0.3in
]



\printAffiliationsAndNotice{\icmlEqualContribution} 

\begin{abstract}
In the vanishing learning rate regime, stochastic gradient descent (SGD) is now relatively well understood. 
In this work, we propose to study the basic properties of SGD and its variants in the non-vanishing learning rate regime. The focus is on deriving exactly solvable results and discussing their implications. The main contributions of this work are to derive the stationary distribution for discrete-time SGD in a quadratic loss function with and without momentum; in particular, one implication of our result is that the fluctuation caused by discrete-time dynamics takes a distorted shape and is dramatically larger than a continuous-time theory could predict. Examples of applications of the proposed theory considered in this work include the approximation error of variants of SGD, the effect of minibatch noise, the optimal Bayesian inference, the escape rate from a sharp minimum, and the stationary covariance of a few second-order methods including damped Newton's method, natural gradient descent, and Adam.
\end{abstract}

\vspace{-2.5mm}
\section{Introduction}
\label{introduction}

Behind the success of deep learning lies the simple optimization methods such as stochastic gradient descent (SGD) \citep{Bottou1999,Sutskever2013,dieuleveut2018bridging,mori2020improved} and its variants \citep{Duchi2011,Flammarion2015,Kingma2017}, which are used for neural network optimization. 
Despite the empirical efficiency of SGD, our theoretical understanding of SGD is still limited. Two types of noises for SGD are studied. When the noise is white, the dynamics is governed by the stochastic gradient Langevin dynamics (SGLD) \citep{Welling2011}. When the noise is due to minibatch sampling, the noise is called the SGD noise or minibatch noise. So far, nearly all the theoretical attempts at understanding the noise in SGD have adopted the continuous-time approximation by assuming a vanishingly small learning rate \citep{Mandt2017,Li2017,Jastrzebski2018,Chaudhari2018,Zhu2019,Xie2020}. This amounts to making an analogy to the diffusion theory in physics \citep{Einstein1905, Kampen2011}, and helps understand some properties of SGD and deep learning such as the flatness of the minima selected by training \citep{Jastrzebski2018,Chaudhari2018,smith2017bayesian, Xie2020}. However, in reality, a large learning rate often leads to qualitatively distinct behavior, including reduced training time and potentially better generalization performance \citep{Keskar2017,Li2019}. The present work is motivated by the fact that the existing continuous theory is insufficient to describe and predict the properties and phenomena of learning in this large learning rate regime. In fact, we will show that the prediction by continuous theory may deviate arbitrarily from the experimental result.

In this work, we study the stationary-state solutions of discrete-time update rules of SGD. The result can be utilized to analyze SGD without invoking the unrealistic assumption of a vanishingly small learning rate. Specifically, our contributions are:
\vspace{-2mm}
\begin{itemize}
    \item 
    We exactly solve the discrete-time update rules for SGD and its variants with momentum in a local minimum to obtain the analytic form of the covariance matrix of the model parameters at asymptotic time.\vspace{-2mm}
    \item We apply our results to various settings that have been studied in the continuous-time limit, such as finding the optimal learning rate in a Bayesian setting, understanding the escape from a sharp minimum, and the approximation error of various variants of SGD.\vspace{-2mm}
    \item Compared with the continuous-time theories, our work requires fewer assumptions and finds significantly improved agreement with experimental results.\vspace{-2mm}
\end{itemize}

In section~\ref{sec: background}, we present the background of this work. Related works are discussed in section~\ref{sec: related works}. In section \ref{sec: theory of SGD}, we derive our main theoretical results for SGD and its momentum variant. In section~\ref{sec: experiments}, we verify our theoretical results experimentally. In section~\ref{sec: applications}, we apply our solution to some well-known problems that have previously been investigated in the continuous-time limit. 
A summary of our results is given in Table~\ref{tab:summary}.

\vspace{-2mm}
\section{Background}\label{sec: background}
\vspace{-1mm}
\begin{figure}[bt!]
    \centering
    \includegraphics[width=0.4\linewidth]{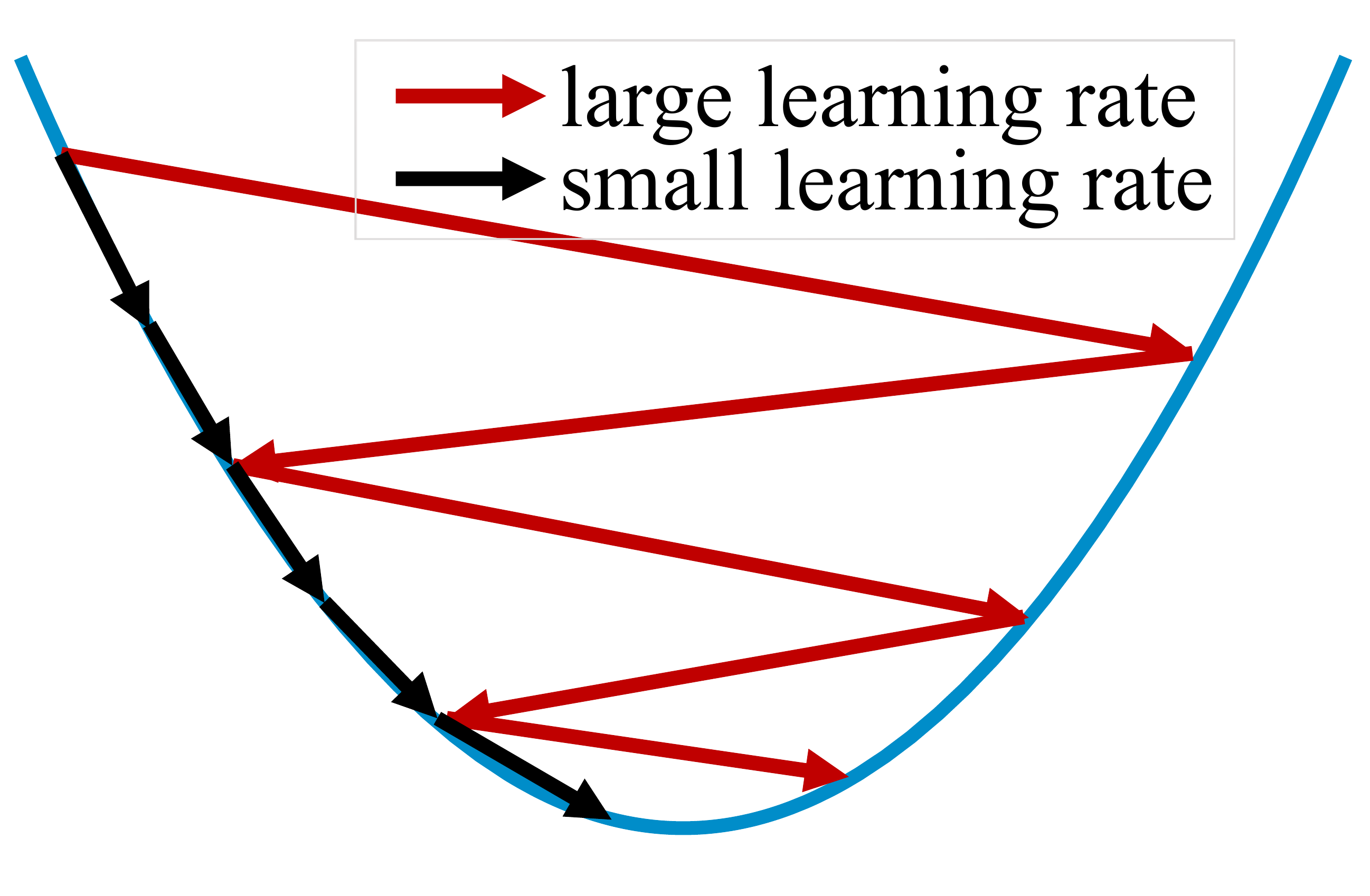}
    \includegraphics[width=0.5\linewidth]{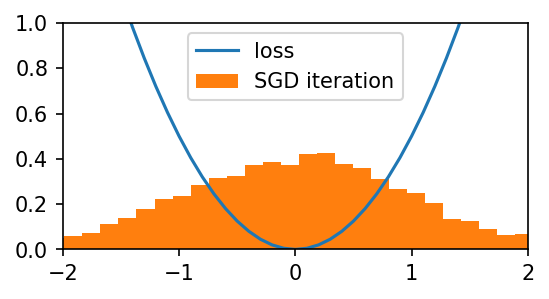}
    \vspace{-1em}
    \caption{
    (Left) Schematic illustrations of the deterministic continuous-time evolution given by $w=w_0 e^{-\lambda kt}$ for small $\lambda$ and the deterministic discrete-time evolution given by Eq.~\eqref{discrete} for $1/k<\lambda<2/k$. The black arrows represent each update according to the continuous-time update rule while the red ones represent the discrete-time updates. (Right) When noise exists, SGD iteration converges not to a point but a distribution. }
    \vspace{-1em}
    \label{fig:large_rate_eg}
\end{figure}


In the presence of noise, it is difficult to give a simple solution to the discrete dynamics. Under the assumptions of a constant Gaussian noise, a quadratic loss function, and an infinitesimal learning rate, theorists approximate the multidimensional update rule by a continuous-time stochastic differential equation \citep{Mandt2017,Zhu2019}, which is a multivariate Ornstein-Uhlenbeck process, 
    $d\mathbf{w}_t = -\lambda K \mathbf{w}_t dt + \lambda C^{\frac{1}{2}}dW_t \label{OU}$, 
where $C$ is the covariance matrix of the noise, $W_t$ represents a standard multidimensional Brownian motion, and $K$ is the Hessian of the local minimum\footnote{In this work, we assume $K$ to be full-rank; this is true in practice since, for applications, it is a standard practice to apply a very small $L_2$ regularization, which amounts to adding a small positive value to all the eigenvalues of $K$, making it full-rank. For more discussion, see Appendix~\ref{app: overparametrization}.}. The stationary covariance matrix, $\Sigma:= \lim_{t\to \infty}\text{cov}[w_t^{\rm T} w_t]$, is found to satisfy the following matrix equation \citep{Kampen2011}:\vspace{-2mm}
\begin{equation}
    \Sigma K + K \Sigma = \frac{\lambda}{1-\mu}C,\label{cont.solution}\vspace{-2mm}
\end{equation}
where $\mu$ is the momentum hyperparameter; when no momentum is used, $\mu=0$. 
Despite the fact that a number of theoretical works has been performed on the basis of the above constant Gaussian noise and the continuous-time approximation \citep{Chaudhari2018,Jacot2018,Zhu2019,Lee2019,Xie2020}, it is clear that the stationary distribution given by Eq.~\eqref{cont.solution} substantially deviates from the true one obtained in experiments\footnote{For example, Figure~1 in \citet{Mandt2017} plots stationary distributions obtained from solving stochastic differential equations. These distributions generally deviate from experimental results. In \citet{gitman2019understanding}, predictions also deviate far from experiments when the learning rate is large.}. The predictions based on these results are qualitatively acceptable only in a small learning rate regime. For a large learning rate, the assumptions simply break down so that the theory becomes invalid.

To intuitively understand how a large learning rate makes a difference, we consider a model with a single parameter $w\in\mathbb{R}$ in a quadratic potential $L=\frac{1}{2}kw^2$ with $k>0$. SGD obeys the dynamical equations as follows: 
\vspace{-1.5mm}
\begin{equation}
    \begin{cases}
        g_t = kw_{t-1} + \eta_{t};\\
        w_t = w_{t-1} - \lambda g_t,
    \end{cases}
\vspace{-1.5mm}
\end{equation}
where $\lambda$ is the learning rate and $\eta_{t}$ is a normal random variable with zero mean and variance $\sigma^2$. When $\sigma^2=0$, the dynamics is deterministic, and the common approach is to assume that $\lambda \ll 1/k$ such that one may take the continuous-time limit of this equation as $\dot{w}=-\lambda k w(t)$, which solves to give $w= w_0 e^{-\lambda k t}$. However, this continuous approximation fails when $\lambda$ is large. To see this, we note that the deterministic discrete-time dynamics solves to give\vspace{-1mm}
\begin{equation}
    w_t = (1-\lambda k )^t w_0, \text{ for } t\in \mathbb{N}^0, \label{discrete}
    \vspace{-1mm}
\end{equation}
which is an exponential decay when $\lambda < \frac{1}{k}$, and, in this region,  the standard continuous-time dynamics is valid with error $O((1-\lambda k)^2)$. On the other hand, when $\lambda > \frac{2}{k}$, the learning is so large that the parameter $w$ will diverge; therefore, the interesting region is when $ \frac{1}{k}<\lambda < \frac{2}{k}$, where the dynamics is convergent, yet a simple continuous approximation fails. See a schematic illustration in Figure~\ref{fig:large_rate_eg}. It is therefore urgent to develop a theory that can handle SGD with a large learning rate.

    
    

\begin{table*}[bt!]
    \centering
    \vspace{-0.7em}
    \caption{\small Comparison of the results of this work with previous results. For notational conciseness, we only show  $\Sigma$ when the noise matrix $C$ commutes with the Hessian $K$. For the \textit{approximation error} panel, the results apply to any $K$ and $C$. $^*$SGD: stochastic gradient descent. $^*$SGDM: stochastic gradient descent with momentum. $^*$QHM: quasi-hypobolic momentum. $^*$DNM: damped newton's method. $^*$NGD: natural gradient descent.
    } 
    \label{tab:summary}
    \vspace{-0.8EM}
    {\small
    \resizebox{\textwidth}{!}{
    \begin{tabular}{c|c|c}
    \hline\hline
    &Previous Work& This Work\\
    \hline
    & { $\Sigma$} & { $\Sigma$} \\
    \hline
    
     SGD &  $\frac{\lambda}{2}K^{-1}C$  
     \citep{Mandt2017}
     &  $\lambda[ K(2I_D-\lambda K)]^{-1}C$     \\
     SGDM &  $\frac{\lambda}{2(1-\mu)}K^{-1}C$ \citep{Mandt2017}
     &  $\lambda\left[ \frac{K}{1+\mu}\left( 2I_D - \frac{\lambda K}{1+\mu}\right) \right]^{-1}\frac{C}{1-\mu^2}$   \\
     QHM  &  $\frac{\lambda}{2(1-\mu)}K^{-1}C + O(\lambda^2)$ \citep{gitman2019understanding} 
     &   $\lambda^2 h(K)^{-1}C$ [Eq.~\eqref{eq:hK}]   \\
     DNM & - & $\frac{1+\mu}{1-\mu}\frac{\lambda}{2(1+\mu)-\lambda}K^{-2}C$ \\
     NGD &- & $\frac{1}{2}K^{-1}\left[Q+\frac{\lambda}{2(1+\mu)}I_D\right]$ [Eq.~\eqref{eq: NGD Sigma}]\\
     Adam &- & $\frac{\lambda^2 (1+c)}{4}I_D$\\
     
    
    \hline\hline
    
    & { Approximation Error} & {Approximation Error} \\
    \hline
    
     SGD &  $\frac{\lambda}{4}{\rm Tr}[C] + \frac{\lambda^2}{8}{\rm Tr}[KC] + O(\lambda^3)$  \citep{gitman2019understanding}
     &  $\frac{\lambda}{2}{\rm Tr}[(2I_D-\lambda K)^{-1}C]$     \\
     SGDM &  $\frac{\lambda}{4(1-\mu)}{\rm Tr}[C] + \frac{\lambda^2}{8}\frac{1}{1-\mu^2}{\rm Tr}[KC] + O(\lambda^3)$ \citep{gitman2019understanding}
     &  $\frac{\lambda}{2(1-\mu)} {\rm Tr} \left[\left(2I_D-\frac{\lambda}{1+ \mu}K\right)^{-1} C \right]$   \\
     QHM  &   $\frac{\lambda}{4}{\rm Tr}[C] + \frac{\lambda^2}{8}\left[1+\frac{2\mu\nu}{1-\mu}\left(\frac{2\mu\nu}{1+\mu}-1\right)\right]{\rm Tr}[KC] + O(\lambda^3)$ \citep{gitman2019understanding}
     &   $\frac{\lambda^2}{2}{\rm Tr}[h(K)^{-1}KC]$ [Eq.~\eqref{generalTr}]   \\
     DNM & - & $\frac{1+\mu}{1-\mu}\frac{\lambda}{4(1+\mu)-2\lambda} {\rm Tr} [K^{-1}C]$\\
     NGD & - & $\frac{1}{4} {\rm Tr} \left[Q+\frac{\lambda}{2(1+\mu)}I_D\right]$\\
     Adam & - & $\frac{\lambda^2 (1+c)}{8}{\rm Tr}[K]$\\
     \hline\hline
    \end{tabular}}}
\vspace{-1.5em}

\end{table*}
\vspace{-2mm}
\section{Related Works}\label{sec: related works}
\vspace{-1mm}




\textbf{Large Learning Rate}. Although continuous-time theory has been the dominant theoretical approach, \citet{Lewkowycz2020} took a step forward in understanding why a larger learning rate may generalize better. They characterized SGD into three regimes according to the learning rate. 
They conjectured that a rather large initial learning rate leads to a ``catapult phase", which helps exploration and often leads to better generalization by converging to a flatter minimum. However, their work is mostly empirical and in the noise-free regime. 
There are more empirical works on large learning rate and generalization. \citet{LeCun2012} found that a large batch size (or a small learning rate) usually leads to reduced generalization performance. \citet{Keskar2017} proposed an empirical measure based on the sharpness of a minimum. They presented numerical evidence that a small learning rate prefers sharp minima that generalize poorly. \citet{goyal2017accurate} showed that setting the learning rate proportional to the minibatch size ensures good generalization, which is crucial for large scale training. There are also works about how noise, the batch size and the learning rate influence the generalization \citep{Hoffer2017,mori2020improved,mori2020deeper}. 
The explanations of why a flatter minimum generalizes better are given by some theories such as the minimum description length theory \citep{rissanen1983}, a Bayesian view of learning \citep{MacKay1992}, and the Gibbs free energy \citep{chaudhari2019entropy}.

\textbf{Escape from Sharp Minimum}. Theoretically understanding why and how SGD converges to flat minima is crucially important \cite{Hochreiter1997}. In fact, among many complexity measures characterizing generalization \citep{dziugaite2017computing,neyshabur2018a, smith2017bayesian,Chaudhari2018,Keskar2017,liang2019fisherrao,nagarajan2019generalization}, the sharpness-based measures have been shown to be the best up to now \citep{Jiang2020}. The continuous approximation is usually adopted to understand how SGD chooses flat minima. \citet{hu2018diffusion} used diffusion theory to show that escape is easier from a sharp minimum than from a flat one. \citet{Wu2018} studied the relationship among learning rate, batch size and generalization from the perspective of dynamical stability. \citet{Jastrzebski2018} used stochastic differential equations to prove that the higher the ratio of the learning rate to the batch size, the flatter minimum will be selected. \citet{Zhu2019} defined the escape efficiency for a minimum and obtained its explicit expression using diffusion theory. They show that anisotropic noise helps escape from sharp minima effectively. A recent work by \citet{Xie2020} calculated the escape rate from a minimum by adopting the formalism of the Kramers escape rate in physics \citep{Kramers1940}. It is shown that SGD with minibatch noise favors flatter minima. Our exact results for a large learning rate make it possible to study the selected flatness and the complexity measures more accurately and may enhance our understanding of deep learning.

\textbf{Bayesian Inference}. SGD has been used for Bayesian inference as well. In Bayesian inference, one assumes a probabilistic model $p(w,x)$ with data $x$ and hidden parameter $w$. The goal is to approximate the posterior $p(w|x)$. Traditionally, stochastic gradient Markov Chain Monte Carlo (MCMC) methods have been used \citep{Welling2011,Ma2015}. A similarity between SGD and MCMC suggests the possibility of SGD being used as approximate Bayesian inference. \citet{Mandt2017} applied SGD to minimize the loss function defined as $-\ln p(w,x)$. They show that one can tune the learning rate such that the Kullback-Leibler (KL) divergence between the learned distribution by SGD and the posterior is minimized. 
This means that SGD can be regarded as an approximate Bayesian inference. However, for a large learning rate, their assumption is no more valid. 

\vspace{-2mm}
\section{Theory of Discrete-Time SGD}\label{sec: theory of SGD}
\vspace{-1mm}

We propose to deal with the discrete-time SGD directly. 
We use $\mathbf{w}\in \mathbb{R}^{D}$ to denote the weights of the model viewed as a vector, and the boldness is dropped when $D=1$. We use capital $K\in \mathbb{R}^{D\times D}$ to denote the Hessian matrix of the quadratic loss function; when $D=1$, we use the lower-case $k$.  We use $\Sigma \in \mathbb{R}^{D\times D}$ to denote the asymptotic covariance matrix of $\mathbf{w}$, and $C \in \mathbb{R}^{D\times D}$ to denote the covariance matrix of a general type of noise $\eta\in \mathbb{R}^{D}$. When the learning rate $\lambda$ is not a scalar but a matrix (sometimes called preconditioning matrix), we use the upper-case letter $\Lambda \in \mathbb{R}^{D\times D}$ instead of $\lambda$. $L$ denotes the training loss function, and $S$ the minibatch size. We use $\mathbb{E}[\cdot]$ to denote the expectation over the stationary distribution of the model parameters. The capital $\rm T$ is used as a superscript to denote matrix transpose and lower case $t$ is used to denote the time step of optimization. Due to space constraint, we leave derivations to Appendix~\ref{app: theory of SGD}. 
\vspace{-2mm}
\subsection{SGD}
\vspace{-1mm}
Consider a general loss function $\mathcal{L}(\mathbf{w'})$ for a general differentiable model with parameters $\mathbf{w}\in \mathbb{R}^{D}$; close to a local minimum, we may expand $\mathcal{L}(\mathbf{w'})$ up to second order in $\mathbf{w}'$. Therefore, close to a local minimum, the dynamics of SGD is governed by a general form of a quadratic potential:\vspace{-2mm}
\begin{equation}
    \mathcal{L}(\mathbf{w'}) \approx \frac{1}{2}(\mathbf{w}'  -\mathbf{w}^*)^{\rm T}K(\mathbf{w}'  -\mathbf{w}^*) := L(\mathbf{w}'),\vspace{-2mm}
\end{equation}
where the Hessian $K$ is a positive definite matrix, and $\mathbf{w}^*$ is a constant vector. One can redefine $\mathbf{w}' - \mathbf{w}^* \to \mathbf{w}$ to obtain a simplified form: $L(\mathbf{w})= \mathbf{w}^{\rm T}K\mathbf{w}/2$. The SGD algorithm with momentum $\mu$ is defined by the update rule\vspace{-1mm}
\begin{equation}
    \begin{cases}
        \mathbf{g}_t = \nabla L(\mathbf{w}_{t-1}) + \eta_{t-1} = K\mathbf{w}_{t-1} + \eta_{t-1};\\
        \mathbf{m}_t = \mu \mathbf{m}_{t-1} + \mathbf{g}_t;\\
        \mathbf{w}_t = \mathbf{w}_{t-1} - \lambda \mathbf{m}_t,
    \end{cases} \label{momentumSGD}
     \vspace{-1mm}
\end{equation}
where the noise $\eta_t$ has a finite covariance $C$, 
and $\mu\in [0,1)$ is the momentum hyperparameter. We first consider the case without momentum, i.e., $\mu=0$. Let $k^{*}$ be the largest eigenvalue of $K$. For any minimum with $\lambda k^* > 2$, the dynamics will diverge, and $\mathbf{w}$ will escape from this minimum. Therefore, we focus on the range of $0<\lambda<\frac{2}{k^{*}}$. This means that the absolute values of all the eigenvalues of $(I_D -\lambda K)$ are in the range of $(0,1)$.


\begin{theorem}\label{thm: SGD} $($Model fluctuations of discrete SGD in a quadratic potential$)$
    Let $\mathbf{w}_t$ be updated according to \eqref{momentumSGD} with $\mu=0$.
    The stationary covariance of $\mathbf{w}$ satisfies
   \vspace{-1mm}
    \begin{equation}
        \Sigma K +  K\Sigma -\lambda K\Sigma K = \lambda C.\label{matrixeq}\vspace{-4mm}
    \end{equation}  
\end{theorem}
Recall that we have shifted the underlying parameter $\mathbf{w}'$ by $\mathbf{w}^*$, and this result translates to that $\mathbb{E}[\mathbf{w}'] =\mathbf{w}^*$, close to the local minimum $\mathbf{w}^*$.

\begin{remark}
The exact condition $\lambda \Sigma K + \lambda K\Sigma -\lambda^2 K\Sigma K = \lambda^2 C$ is different from the classical result obtained from a continuous Ornstein-Uhlenbeck process, which has $\lambda \Sigma K + \lambda K\Sigma = \lambda^2 C$. This suggests that the approximation of a discrete-time dynamics with a continuous-time one in \citet{Mandt2017} can be thought of as a $\lambda$-first-order approximation to the true dynamics. This approximation incurs an error of order $\mathcal{O}(\lambda^2)$. The error becomes significant or even dominant as $\lambda $ gets large.
\end{remark}
\begin{remark}
 This theorem only requires the existence of a finite stationary noise covariance, in contrast to the unrealistic assumption of constant Gaussian noise by \citet{Mandt2017}. In Appendix~\ref{app sec: non gaussian noise}, we show that the agreement of the theory is as good when the noises are non-gaussian.
\end{remark}


\vspace{-3.5mm}
\subsection{SGD with Momentum}\label{sec: sgd with momentum}
\vspace{-1mm}
We now consider the case with arbitrary $\mu\in [0,1)$ in Eq.~\eqref{momentumSGD}. 

\begin{theorem}\label{theo: main theorem with momentum}
    $($Stationary distribution of discrete SGD with momentum$)$
    Let $\mathbf{w}_t$ be updated according to \eqref{momentumSGD} with arbitrary $\mu\in [0,1)$.
    Then $\Sigma$ is the solution to\vspace{-2mm}
    \begin{align}
    &\underbrace{- \frac{1+\mu^2}{1-\mu^2}\lambda^2 K\Sigma K + \frac{\mu}{1-\mu^2}\lambda^2( K^2\Sigma +\Sigma K^2)}_{\text{discrete-time}} \nonumber\\
    &+ \underbrace{(1-\mu)\lambda (K\Sigma + \Sigma K )}_{\text{continuous-time}}   = \lambda^2 C. \label{matrixeqmomentum}
    \vspace{-2mm}
    \end{align}
If the noise is Gaussian, then the stationary state distribution of  $\mathbf{w}$ is $
     \mathbf{w} \sim \mathcal{N}(0, \Sigma).$
\end{theorem}\vspace{-2mm}
 We examine the above result \eqref{matrixeqmomentum} with two limiting examples. First, if there is no momentum, namely $\mu=0$, we recover the previous result \eqref{matrixeq} without momentum. Next, if $\lambda K \ll 1$, neglecting the $\mathcal{O}((\lambda K)^2)$ terms recovers the result of SGD with momentum under the continuous approximation \eqref{cont.solution} derived by \citet{Mandt2017}. When $C$ commutes with $K$, the above matrix equation can be solved explicitly.
 
\begin{corollary}\label{corollary: momentum commuting}
    Let $[C,K] = 0$. Then\vspace{-2mm}
    \begin{equation}
        \Sigma = \left[ \frac{\lambda K}{1+\mu}\left( 2I_D - \frac{\lambda K}{1+\mu}\right) \right]^{-1}\frac{\lambda^2 C}{1-\mu^2}.\vspace{-2mm}
    \end{equation}
\end{corollary}
We can obtain a more general result when the learning rate is a matrix.
\begin{theorem}\label{thm: preconditioning matrix eq}
If the learning rate is a positive definite preconditioning matrix $\Lambda$, then $\Sigma$ satisfies\vspace{-2mm}
\begin{align}
    &- \frac{1+\mu^2}{1-\mu^2}\Lambda K\Sigma K \Lambda+ \frac{\mu}{1-\mu^2}(\Lambda K\Lambda K\Sigma +\Sigma K\Lambda K\Lambda) \nonumber\\ 
    &+ (1-\mu) (\Lambda K\Sigma + \Sigma K \Lambda) = \Lambda C\Lambda. \label{eq: preconditioning matrix eq}
\end{align}
\end{theorem}
\vspace{-2mm}
Note that the matrix $\Lambda$ does not necessarily commute with $K$. We consider an application of this general result to understanding second-order methods in section~\ref{sec: second order}.

\vspace{-2mm}
\subsection{Two Typical Kinds of Noise}
\vspace{-1mm}

As mentioned, two specific types of noise are of particular interest for machine learning practices. The first is a multidimensional white noise: $\eta\sim \mathcal{N}(0,\sigma^2 I_D)$, where $\sigma$ is a positive scalar. The the second type of noise is of Gaussian type with a covariance proportional to the Hessian, which is approximately equal to the noise caused by a minibatch gradient descent algorithm \citep{Zhu2019, Xie2020}: $C=aK$, where $a$ is a constant scalar, because\vspace{-7mm}

{\small
\begin{equation}
    C(\mathbf{w})
    \approx \frac{1}{S}\frac{1}{N}\sum_{n=1}^{N}\nabla L(\mathbf{w},x_n)\nabla L(\mathbf{w},x_n)^{\rm T} := \frac{1}{S}J(\mathbf{w})\approx \frac{1}{S}K, \nonumber\vspace{-2mm}
\end{equation}}%
where the first approximation is due to the fact that noise dominates the dynamics in the vicinity of a minimum, and the second approximation is according to the similarity between the Fisher information $J(\mathbf{w})$ and the Hessian $K$ near a minimum. This approximation is somewhat crude although it is often employed.

To be more general, one might wish to mix an isotropic Gaussian noise with the minibatch noise, namely $C= \sigma^2 I_D + aK $. The following theorem gives the distribution.
\begin{theorem} \label{thm: typical noise}
    Let $C =\sigma^2 I_D + a(\lambda)K  $ and $\mu=0$. Then the stationary distribution of $\mathbf{w}$ is\vspace{-1mm}
    \begin{equation}
        \mathcal{N}\left(0, \lambda (\sigma^2I_D + a K)[K(2I_D -\lambda K)]^{-1}\right).\vspace{-2mm}
    \end{equation}
\end{theorem}
Notice that, in this case, $C$ commutes with $K$. From this result, we can derive the following two special cases by setting $\sigma^2=0$ or $a=0$.
\begin{corollary}\label{cor: typical noise 1}
    Let $\sigma^2=0$. Then 
        $\Sigma = a\lambda  (2I_D -\lambda K)^{-1}$.
\end{corollary}
\begin{corollary}\label{cor: typical noise 2}
    Let $a=0$. Then $ \Sigma = \sigma^2\lambda [K (2I_D -\lambda K)]^{-1}$.
\end{corollary}
Interestingly, when $\lambda\ll 1$, the minibatch noise results in isotropic fluctuations, independent of the underlying geometry; the discrete time steps, however, causes fluctuations in the direction of large eigenvalues of the Hessian.
        


\vspace{-2mm}
\subsection{Approximation Error of SGD}
\vspace{-1mm}
We note that one important quantity for measuring the approximation error of SGD is ${\rm Tr} [K\Sigma]$, because the expectation of a quadratic loss is\vspace{-2mm}
\begin{equation}
     L_{\rm train}:=\mathbb{E}\left[\frac{1}{2}\mathbf{w}^{\rm T}K\mathbf{w} \right]=\frac{1}{2}{\rm Tr} [K\Sigma],\vspace{-2mm}
\end{equation}
where the expectation is taken over the stationary distribution of $\mathbf{w}$.

\begin{theorem}\label{thm: SGDM train error}$($Approximation error for discrete SGD with momentum$)$ Let the noise covariance $C$ and Hessian $K$ be any positive definite matrix. Then the training error for SGD with momentum is\vspace{-2mm}
\begin{equation}
   L_{\rm train} = \frac{\lambda}{4(1-\mu)} {\rm Tr} \left[\left(I_D-\frac{\lambda}{2(1+ \mu)}K\right)^{-1} C \right].\vspace{-2mm}
\end{equation}
\end{theorem}
\vspace{-2mm}
This result suggests that a larger eigenvalue in the Hessian causes larger training error. Also, compared with the continuous result: $L_{\rm train} = \frac{\lambda }{4(1-\mu)} {\rm Tr}[C]$, the discrete theory results in a larger approximation error.

\vspace{-2mm}
\subsection{Necessary Stability Condition}
\vspace{-1mm}
The main result in Theorem~\ref{theo: main theorem with momentum} also suggests a condition for the convergence of SGD. In order for a stationary distribution to exist at convergence, the covariance $\Sigma$ needs to exist and be positive definite, and this is only possible when $\left(2I_D - \frac{\lambda}{1+\mu} K \right)$ is positive definite. This condition reflects the fact that discrete-time SGD becomes \textit{ill-conditioned} as $\lambda$ increases, and so the continuous-time approximation becomes less valid. Also, an important implication is that using momentum may mitigate the ill-conditioning of the large learning rate, but only up to a factor of $1+\mu<2$, before the momentum causes another divergence problem due to the term $\frac{1}{1-\mu}$. Therefore, when momentum is used, the necessary condition for convergence becomes $\lambda k^* \leq 2(1+\mu)<4$. We also comment that this is only a necessary condition for stability; the sufficient condition of stability is highly complicated and we leave this to a future work. 

\vspace{-2mm}
\subsection{Regularization Effect of a Finite Learning Rate}
\vspace{-1mm}
For the continuous dynamics, the stationary distribution is known to obey the Boltzmann distribution, $P(\mathbf{w}) \sim \exp[-L(\mathbf{w})/T]$ for some scalar $T$ determined by the noise strength \citep{Landau1980, Chiyuan_SGD}. This implies that, close to a local minimum $\mathbf{w}_*$, the stationary distribution is approximated by\vspace{-1mm}
\begin{equation}
    P(\mathbf{w}| \mathbf{w}^*) \sim \exp[-(\mathbf{w} -\mathbf{w}^*)^{\rm T} K (\mathbf{w} -\mathbf{w}^*)/T].\vspace{-1mm}
\end{equation}
Comparing with the standard continuous-time solution \eqref{cont.solution}, we see that this corresponds to an isotropic noise, namely $C=2T I_D$, and $L_{\rm eff}:= -T \log P(\mathbf{w}| \mathbf{w}^*)$ may be defined as an effective loss function in analogy with an effective free energy in theoretical physics.

For discrete-time SGD, however, the stationary distribution has an additional term to leading order:
\vspace{-1mm}
{\scriptsize
\begin{align}
    P_{\text{d}}(\mathbf{w} |  \mathbf{w}^*) 
    = \exp\left[-\frac{1}{T}(\mathbf{w} -\mathbf{w}^*)^{\rm T} \left(K + \frac{\lambda}{2} K^2\right) (\mathbf{w} -\mathbf{w}^*) + O(\lambda^2)\right]\nonumber.
\end{align}}
\vspace{-1mm}
This implies a different form of the effective loss function:\vspace{-6mm}

{\small
\begin{equation}
     L_{\rm deff} 
     \sim (\mathbf{w} -\mathbf{w}^*)^{\rm T} K (\mathbf{w} -\mathbf{w}^*) + \frac{\lambda}{2} (\mathbf{w} -\mathbf{w}^*)^{\rm T} K^2 (\mathbf{w} -\mathbf{w}^*),\nonumber \vspace{-2mm}
\end{equation}}%
where the first term is the same as the continuous-time loss function, while the second term emerges as a discrete-time contribution due to a large learning rate. In particular, it encourages $\mathbf{w}$ to have a smaller norm around the minimum $\mathbf{w}_*$ in the kernel $K^2$. Therefore, to first order, the effect of the large learning rate can be understood as an effective weight decay term that encourages a smaller norm.

\subsection{State-Dependent Noise}\label{sec: state-dept noise}
In reality, the minibatch noise depends on the current value of $\mathbf{w}$, and the noise is, therefore, not constant \citep{Simsekli2019tailindex,simsekli2020hausdorff,hodgkinson2020multiplicative,meng2020dynamic,ziyin2021minibatch,mori2021logarithmic}. This means that the noise covariance is a function of the parameters covariance, namely $C=C(\Sigma)$. The general Theorem~\ref{theo: main theorem with momentum} is still applicable to this setting, and $\Sigma$ can be solved by setting $C=C(\Sigma)$, provided that $\Sigma$ exists. One example of state-dependent noise is given in section~\ref{sec: second order} where the noise covariance is set to $C\approx K\Sigma K$. After the publication of this work, we notice that \citet{ziyin2021minibatch} has applied our formalism to study the exact shape and strength of SGD noise when minibatch sampling is used, and we refer the readers to this work for a more detailed study of the state-dependent noise in SGD.

\section{Experiments}\label{sec: experiments}
\vspace{-1mm}
\begin{figure}[bt!]
    \centering
    \begin{subfigure}{0.49\columnwidth}
    \includegraphics[width=\linewidth]{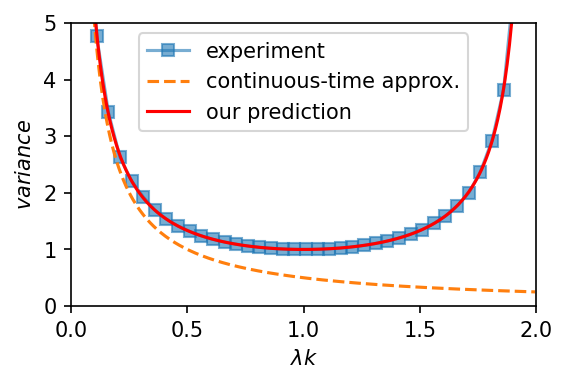}
    \vspace{-2em}
    \caption{White noise}
    \end{subfigure}
    \hfill
    \begin{subfigure}{0.49\columnwidth}
    \includegraphics[width=\linewidth]{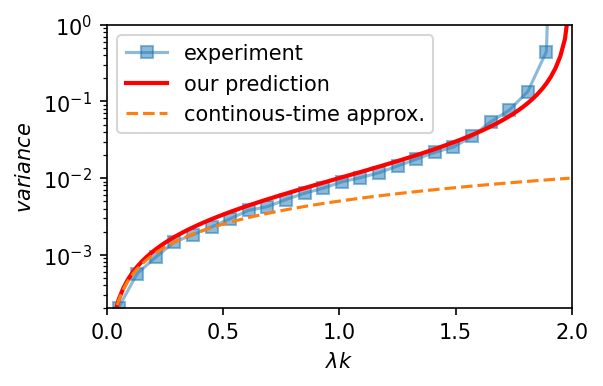}
    \vspace{-1.5em}
    \caption{Minibatch noise}
    \end{subfigure}
    \vspace{-1.0em}
    \caption{Comparison for a 1d example for predicting $\Sigma$. We see that the continuous prediction is only acceptable for $\lambda k<1$, while the discrete theory holds well for all allowed $\lambda$. (a) White noise with strength $\sigma^2=1$. (b) Minibatch noise.}
    \vspace{-1em}
    \label{fig:1d noise}
\end{figure}
\textbf{Gaussian Noise}. We first consider the case when $w\in \mathbb{R}$ is one-dimensional. The loss function is $L(w) = \frac{1}{2}kw^2$ with $k=1$. In Figure~\ref{fig:1d noise}(a), we plot the variance of $w$ after $1000$ training steps from $10^4$ independent runs. We compare the prediction of Corollary~\ref{corollary: momentum commuting} with that of the continuous-time approximation in \citet{Mandt2017}. We see that the proposed theory agrees excellently with the experiment, whereas the standard continuous-time approximation fails as $\lambda$ increases. Moreover, the continuous-time approximation fails to predict the divergence of the variance at $\lambda k\to 2$, whereas the discrete theory captures this very well.
Now, we consider a multidimensional case. We set the loss function to be
 $L(\mathbf{w}) = \frac{1}{2} \mathbf{w}^{\rm T} K \mathbf{w}$.
For visualization, we choose $D=2$, and we set the eigenvalues of the Hessian matrix to be $1$ and $0.1$, and plot the fluctuation along the eigenvectors of this Hessian matrix. See Figures~\ref{fig:2d noise}(a)-(c). As before, we compare the discrete-time results with the theoretical predictions of \citet{Mandt2017}. After diagonalization, the continuous-time dynamics predicts $\Sigma=\rm{diag}(\lambda/2, \lambda/0.2)$, whereas the discrete theory predicts $\Sigma=\rm{diag}\left(\frac{\lambda}{2-\lambda}, \frac{\lambda}{0.1(2-0.1\lambda)}\right)$. The proposed theory exhibits no noticeable deviation from the experiment and successfully predicts a distortion along the direction with a large eigenvalue in the Hessian. In comparison, the continuous-time approximation always underestimates the fluctuation in the learning, and the discrepancy is larger as the learning rate gets larger; the prediction of the continuous-time theory can deviate arbitrarily far from the experiment as $\lambda$ gets close to the divergence value.

\begin{figure}[bt!]
    \centering
    \begin{subfigure}{0.32\columnwidth}
    \includegraphics[width=\linewidth]{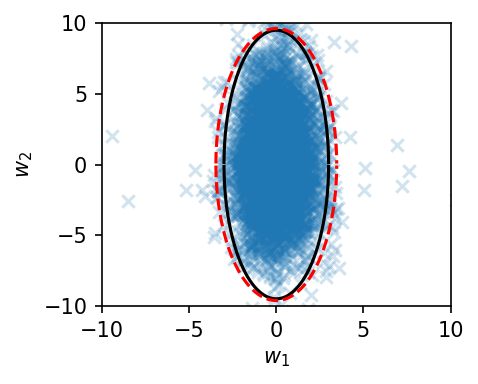}
    \caption{$\lambda=0.5$}
    \end{subfigure}
    \hfill
    \begin{subfigure}{0.32\columnwidth}
    \includegraphics[width=\linewidth]{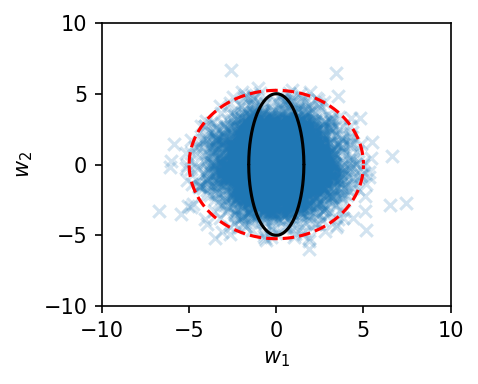}
    \caption{$\lambda=1.8$}
    \end{subfigure}
    \hfill
    \begin{subfigure}{0.32\columnwidth}
    \includegraphics[width=\linewidth]{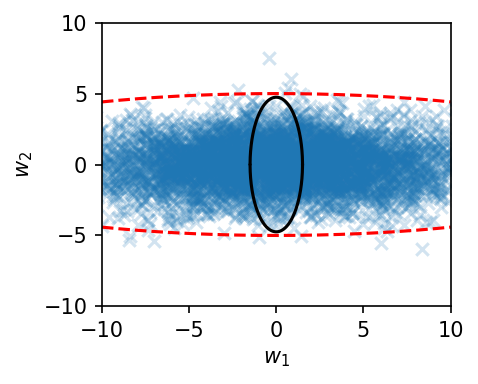}
    \caption{$\lambda=1.99$}
    \end{subfigure}

    \centering
    \begin{subfigure}{0.32\columnwidth}
    \includegraphics[width=\linewidth]{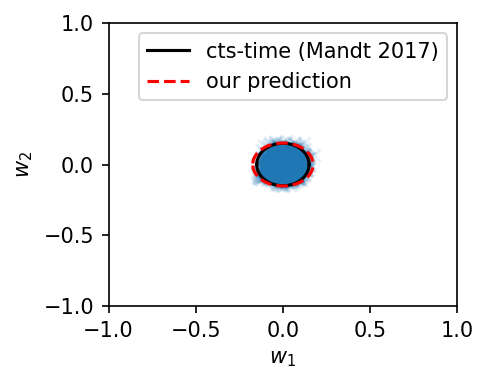}
    \caption{$\lambda=0.5$}
    \end{subfigure}
    \begin{subfigure}{0.32\columnwidth}
    \includegraphics[width=\linewidth]{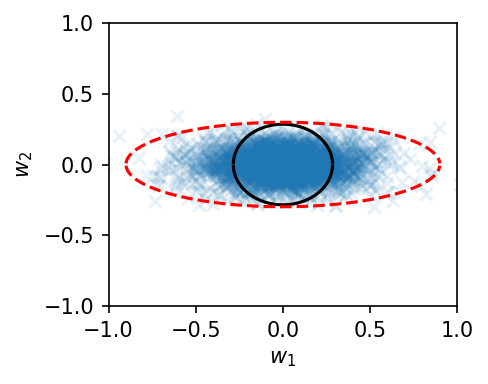}
    \caption{$\lambda=1.8$}
    \end{subfigure}
    \begin{subfigure}{0.32\columnwidth}
    \includegraphics[width=\linewidth]{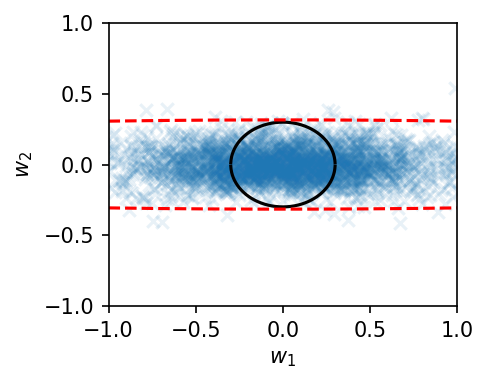}
    \caption{$\lambda=1.99$}
    \end{subfigure}
    \vspace{-1.0em}
    \caption{Comparison between our prediction and that of \citet{Mandt2017}  at $3\sigma$ confidence interval, i.e., $99\%$ of data points lie within the boundary of the theory. We see that our prediction agrees well with the experiments across all levels of the learning rate, whereas the prediction by \citet{Mandt2017} applies only at small $\lambda$. (a)-(c) White noise. (d)-(f) Minibatch noise. }
    \label{fig:2d noise}
    \vspace{-6mm}
\end{figure}

\textbf{Minibatch Noise}. For minibatch noise,
we solve a linear regression task with the loss function 
    $L(w)=\frac{1}{N}\sum_{i}^N(\mathbf{w}^{\rm T}x_i - y_i)^2$,
where $N=1000$ is the number of data points; for the 1d case, the data points $x_i$ are sampled independently from a normal distribution $\mathcal{N}(0, 1)$; $y_i = \mathbf{w}^*x_i + \epsilon_i$ with a constant but fixed $\mathbf{w}^*$, $\epsilon_i$ are noises, also sampled from a normal distribution. For sampling of minibatches, we set the batch size $S=100$. 
The theoretical noise covariance matrix is approximated by $C\approx  K/ S$. See Figure~\ref{fig:1d noise}(b) for the 1d comparison. We see that the proposed theory agrees much better with the experiment than the continuous theory, both in trend and in magnitude. We also compare the predicted distribution for $D=2$. Here, the data points $x_i$ are sampled from $\mathcal{N}(0, \rm{diag}(1, 0.1))$, which makes the expected Hessian equal to $\rm{diag}(1, 0.1)$. 
See Figures~\ref{fig:2d noise}(d)-(f) for the illustration. Again, an overall agreement with the experiment is much better for the proposed theory. We notice that the prediction by \citet{Mandt2017} and the discrete theory agree well with each other when $\lambda$ is small. 
Interestingly, the proposed theory slightly overestimates the variance of the parameters. This suggests the limitation of the commonly used approximation, $C\sim K$, of the minibatch noise. It is possible to treat the minibatch noise in discrete-time regime rigorously in the proposed framework \citep{ziyin2021minibatch,mori2021logarithmic}.


\textbf{SGD with Momentum}. For white noise, we set $L(w)= kw^2/2$ as before. In Figure~\ref{fig:1d white noise with momentum}(a), we plot the case with $\lambda k >2$, where the dynamics will diverge if no momentum is present. The experiment does show this divergence at the value of $\mu\to \lambda k/2 -1$ implied by the necessary stability condition, in contrast to the continuous-time theory that predicts no divergence. 
In Figure~\ref{fig:1d white noise with momentum}(b), we show the experiments with minibatch noise with the same $\lambda$. The loss is the same as that of the minibatch noise. The predicted theory agrees better than the continuous-time approximation. On the other hand, the agreement becomes worse as the fluctuation in $w$ becomes large. This suggests the limitation 
of the commonly used approximation of minibatch noise, i.e., $C\sim H(w)=K$. More experimental results with a smaller learning rate are included in Appendix~\ref{app: SGDM exp}.

\begin{figure}[bt!]
    \centering
    \begin{subfigure}{0.49\columnwidth}
    \includegraphics[width=\linewidth]{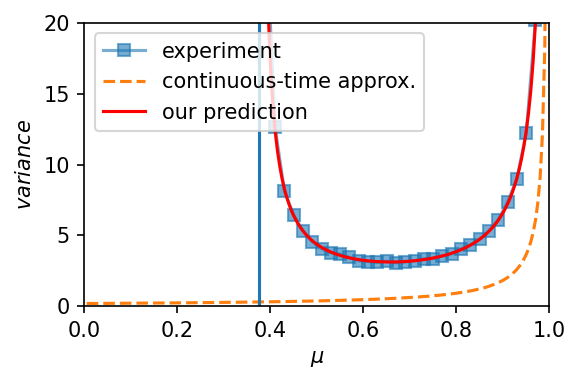}
    \vspace{-2em}
    \caption{White noise}
    \end{subfigure}
    \hfill
    %
    \begin{subfigure}{0.49\columnwidth}
    \includegraphics[width=\linewidth]{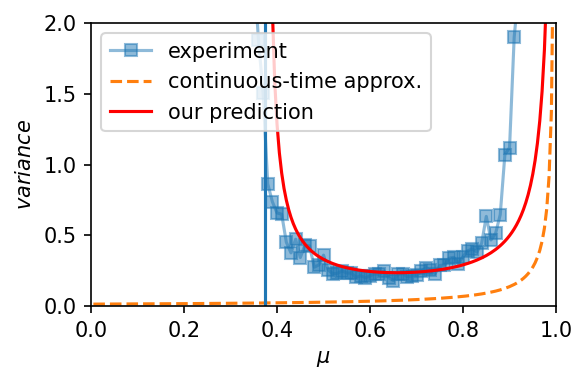}
    \vspace{-2em}
    \caption{Minibatch noise}
    \end{subfigure}
    \vspace{-1.0em}
    \caption{Comparison between the continuous-time theory and the discrete-time theory with momentum. $\lambda k =2.75$ for both experiments. (a) White noise. The vertical line shows $\mu= \frac{\lambda k -2}{2}$, where the variance is predicted to diverge. 
    (b) Minibatch noise.} 
    \label{fig:1d white noise with momentum}
    \vspace{-1em}
\end{figure}

\vspace{-3mm}
\section{Applications}\label{sec: applications}
\vspace{-1mm}
In this section, we apply our exact solution to some important problems studied in the recent literature, such as the derivation of the Bayesian optimal learning rate in section~\ref{sec: optimal learning rate}, and escape from a sharp minimum\footnote{In the existing machine learning literature, there are two definitions for the escape rate. The first was introduced by \citet{Zhu2019}. We call it the ``escape rate of the first kind" (in physics, this rate is more properly called the ``thermalization rate" rather than the ``escape rate" ). The second type is the familiar ``Kramers escape rate" \citep{Kramers1940}.}. 
Previously, these problems have been solved with the continuous-time approximation in a quadratic loss. 
It is therefore of interest to investigate how the exact solution corrects the established results in the large-$\lambda$ regime. The detailed derivations are included in Appendix~\ref{app: applications}.

\vspace{-2mm}
\subsection{Optimal Bayesian Learning Rate}\label{sec: optimal learning rate}
\vspace{-1mm}
We follow the setting of \citet{Mandt2017} for analyzing SGD as an approximate Bayesian inference algorithm. We assume a probabilistic model $p(\mathbf{w},\mathbf{x})$ with $N$-dimensional data $\mathbf{x}$. Our goal is to approximate the posterior\vspace{-2mm}
\begin{equation}
    p(\mathbf{w}|\mathbf{x})=\exp[\ln p(\mathbf{w},\mathbf{x})-\ln p(\mathbf{x})].
    \vspace{-1mm}
\end{equation}
The loss function is defined as
    $L(\mathbf{w}):=\frac{1}{N}\sum_{n=1}^{N}l_n (\mathbf{w})$,
where
    $l_n (\mathbf{w}):=-\ln p(x_n |\mathbf{w})-\frac{1}{N}\ln p(\mathbf{w})$.
The posterior is approximately Gaussian:\vspace{-2mm}
\begin{equation}
    f(\mathbf{w})\propto \exp\left\{-\frac{N}{2}\mathbf{w}^{\rm T}K\mathbf{w}\right\}.\label{posterior}\vspace{-2mm}
\end{equation}
\begin{theorem}\label{thm: Bays}
If the noise covariance is $C=\frac{N-S}{NS}K$, the optimal learning rate for minimizing the KL divergence $D_{\rm KL}(q||f)$ between the SGD stationary distribution $q(\mathbf{w})$ in Theorem~\ref{thm: SGD} and the posterior \eqref{posterior} is the solution to\vspace{-1mm}
\begin{equation}
    \frac{N-2S}{S}\sum_{i=1}^{D}\frac{k_i}{2-\lambda k_i}+\frac{N-S}{S}\lambda\sum_{i=1}^{D}\frac{k_i^2}{(2-\lambda k_i)^2}=\frac{D}{\lambda}, \label{eq: DKL=0}\vspace{-1mm}
\end{equation}
where $k_i$ is the $i$-th eigenvalue of the Hessian $K$.
\end{theorem}
\vspace{-2mm}

This relation constitutes a general solution to this problem, which can be solved numerically. The result by \citet{Mandt2017} can be seen as a solution to the above equation after ignoring non-linear terms in $\lambda$, which gives $\lambda^{*}_{c}=2\frac{S}{N}\frac{D}{\text{Tr}[K]}$ under the assumptions that $S\ll N$ and $\lambda k\ll 1$. With increasing $\lambda k$, one requires increasingly higher-order corrections from Eq.~\eqref{eq: DKL=0}.

\vspace{-2mm}
\subsection{Escape Rate of the First Kind }\label{subsec: escape}
\vspace{-1mm}
Now, we investigate the effect of discreteness on the escape rate from a sharp minimum. The first indicator for the escape rate, called the escaping efficiency, is proposed by \citet{Zhu2019} as\vspace{-2mm}
\begin{equation}
    E(t):=\mathbb{E}[L(\mathbf{w}_t)-L(\mathbf{w}_0)],\vspace{-2mm}
\end{equation}
where $\mathbf{w}_0$ is the exact minimum and $t$ is a fixed time. This indicator qualitatively characterizes the ability of escape from the minimum $\mathbf{w}_0$. It is related to the escape probability via the Markov inequality $
    P(L(\mathbf{w}_t)-L(\mathbf{w}_0)\ge \delta)\le \frac{E}{\delta}$,
for $\delta>0$, if $\delta=\Delta L$ is the height of the potential barrier.

\begin{theorem}\label{thm: discrete escape efficiency}$($Escaping efficiency from a sharp minimum$)$ 
Let the algorithm be updated according to Eq.~\eqref{momentumSGD} with $\mu=0$. Then the escaping efficiency is\vspace{-2mm}
\begin{equation}
    E_d= \frac{\lambda}{4} \text{Tr}\left[\left(I_D -\frac{\lambda K}{2}\right)^{-1}\left[I_D - (I_D-\lambda K)^{2t}\right]C\right]. \label{eq: efficiency}\vspace{-2mm}
\end{equation}
\end{theorem}
The subscript $d$ indicates discrete-time. In comparison,  the escaping efficiency calculated from the continuous-time approximation is given by \citet{Zhu2019}\vspace{-2mm}
\begin{equation}
    E_c =\frac{\lambda}{4}\text{Tr}\left[\left(I_D-e^{-2\lambda K t}\right)C\right],
    \vspace{-2mm}
\end{equation}
where the subscript $c$ indicates continuous-time. The two results can be easily compared in two limiting cases. First, in the short-time limit, the continuous-time theory predicts $E_c=\frac{t\lambda^2}{2}\text{Tr}[KC]$, which coincides with the single-step $t=1$ result given by the discrete theory.
Second, when $t\gg 1$, The continuous indicator becomes Hessian-independent: $E_c = \frac{\lambda}{4}\text{Tr}[C]$,
whereas the discrete result still depends on the curvature: $E_d=\frac{\lambda}{2}\text{Tr}\left[(2I_D -\lambda K)^{-1}C\right]$. The conclusion that a flatter minimum relates to a smaller escaping efficiency still holds. If we take the small-$\lambda$ limit, it recovers the trivial continuous result. In Figure~\ref{fig:first escape rate}, we compare the prediction of Eq.~\eqref{eq: efficiency} with the continuous theory. We set $C=I_D$, $K=I_D$ and compare at two different levels of learning rate. The result is averaged over $50000$ runs. We see that our solution agrees with the experiment perfectly, while the continuous theory significantly underestimates the escape rate and fails at a large learning rate.
 
 \begin{figure}[bt!]
     \centering
     \begin{subfigure}{0.42\columnwidth}
     \includegraphics[width=\columnwidth]{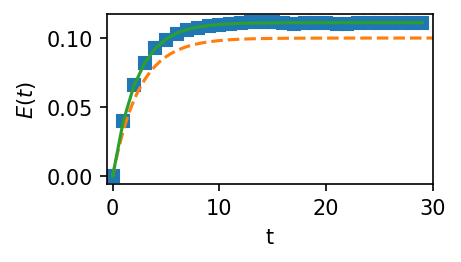}
      \vspace{-2em}
     \caption{$\lambda=0.2$}
     \end{subfigure}
     \begin{subfigure}{0.42\columnwidth}
     \includegraphics[width=\columnwidth]{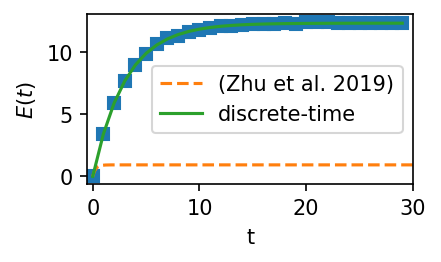}
      \vspace{-2em}
     \caption{$\lambda=1.85$}
     \end{subfigure}
     \vspace{-1em}
     \caption{Theoretical prediction of the escape efficiency $E(t)$ vs. experiment. The continuous-time prediction follows from \citet{Zhu2019}. The discrete-time prediction agrees very well with the experiment, whereas the continuous-time prediction shows significant deviations as $\lambda$ increases and completely fails at $\lambda$ is close to $2$. }
    \label{fig:first escape rate}
     \vspace{-1em}
 \end{figure}
  The following corollary shows that the discrete theory predicts a higher escape probability than the continuous one.
 \begin{corollary}\label{cor: Ed > Ec} $\forall\ 0< \lambda <2/k^{*}$ and $t\ge 0$, $E_d\ge E_c$.
 \end{corollary}


\vspace{-2mm}
We then investigate the effect of anisotropic noise on the escape efficiency as in \citet{Zhu2019}. We consider different structure of noise with the same magnitude $\text{Tr}[C]$. We define an \textit{ill-conditioned} Hessian $K$ as its descending ordered eigenvalues $k_1\ge k_2\cdots \ge k_D>0$ satisfy  $k_{l+1},k_{l+2},\dots ,k_{D}<k_1 D^{-d}$ for some constant $l\ll D$ and $d>1/2$. We assume that $C$ is \textit{aligned with} $K$. Let $u_i$ be the corresponding unit eigenvector of eigenvalue $k_i$. For some coefficient $a>0$, we have $
    u_1^{\text{T}}Cu_1\ge ak_1 \text{Tr}[C]/\text{Tr}[K]$.
This is true if the maximal eigenvalues of $C$ and $K$ are aligned in proportion, namely $c_1 /\text{Tr}[C] \ge a_1 k_1/\text{Tr}[K]$.
\begin{theorem}\label{thm: efficiency ratio}
Under the conditions of an ill-conditioned Hessian and a noise covariance that aligns with the Hessian, the ratio between the escaping efficiencies of an anisotropic noise and its isotropic version $\bar{C}:=\frac{\text{Tr}[C]}{D}I_D$ is\vspace{-2mm}
\begin{equation}
    \frac{\text{Tr}[KC]}{\text{Tr}[K\bar{C}]}=\mathcal{O}\left(aD^{2d-1}\right). \label{eq: efficiency ratio}\vspace{-2mm}
\end{equation}
\end{theorem}
\begin{remark}
This result shows that the previous understanding that the anisotropy in noise may help escape from a sharp minimum still holds in a discrete-time regime. Therefore, the qualitative result in \citet{Zhu2019} still holds when a large learning rate is used.
\end{remark}

\vspace{-3mm}
\subsection{Escape from Sharp Minima (Kramers Problem)}\label{sec: escape 2}
\vspace{-1mm}
In physics, the Kramers escape problem \citep{Kramers1940} concerns the approximate mean time for a particle confined in a local minimum of a potential $L(w)$ to escape across the potential barrier. For continuous-time dynamics, the standard approach to calculating this Kramers rate (or time) \citep{Hanggi1990,Kampen2011} is to employ the Fokker-Planck equation for the distribution $P_c(w,t)$ ($c$ for continuous-time)\vspace{-2mm}
\begin{equation}
    \frac{\partial P_c(w,t)}{\partial t}=-\nabla\cdot J(w,t), \label{eq: FP}\vspace{-2mm}
\end{equation}
where the probability current is defined as $
    J(w,t):=-\lambda P_c(w,t)\nabla L(w) - \mathcal{D}\nabla P_c(w,t)
$. Here $\mathcal{D}:=\frac{1}{2}C$ is the diffusion matrix and $T$ is the effective ``temperature". The mean escape rate is defined as\vspace{-2mm}
\begin{equation}
    \gamma:=\frac{P(w\in V_a)}{\int_{\partial a}J\cdot dS},\label{eq: def of Kramers}\vspace{-2mm}
\end{equation}
where $P(w\in V_a)$ is the probability of a particle being inside the well $a$, and $\int_{\partial a}J\cdot dS$ is the probability flux through the boundary of the well $a$. 
For illustration, see Figure~\ref{fig:escape}(a). In continuous theory, the probability inside the well $a$ can be approximated by $1$ because the distribution lies almost within the well. However, the discrete theory predicts larger fluctuations such that the distribution spreads out with large $\lambda$. By making proper approximations, we improve the result on the Kramers rate for the discrete SGD.

\begin{theorem}\label{thm: Kramers}
Let $k_a$ and $k_b$ be the local Hessian at the local minimum $a$ and the barrier top $b$, respectively. Suppose $l$  is a midpoint on the most probable escape path between $a$ and $b$ such that $k(\mathbf{w})\approx k_a$ in the path $a\to l$ and $k(\mathbf{w})\approx k_b$ in $l\to b$. The approximate Kramers escape rate from a local minimum $a$ for the discrete-time SGD is\vspace{-2mm}
\begin{multline}
    \gamma\approx\frac{1}{2\pi}|k_b|\sqrt{\frac{2}{2-\lambda k_a}}\textup{erf} \left(\sqrt{\frac{S(2-\lambda k_a)\Delta L}{\lambda k_a}}\right)\\
    \times\exp\left[-\frac{2S\Delta L}{\lambda}\left(\frac{l(1-\lambda k_a /2)}{k_a}+\frac{1-l}{|k_b|}\right)\right],\vspace{-2mm}
\end{multline}
where $\textup{erf}(z)$ is the error function.
\end{theorem}
\vspace{-2mm}
To compare, the mean escape rate obtained from the continuous-time theory \citep{Xie2020} is\vspace{-2mm}
\begin{equation}
    \gamma_c=\frac{1}{2\pi}|k_b|\exp\left[-\frac{2S\Delta L}{\lambda}\left(\frac{l}{k_a}+\frac{1-l}{|k_b|}\right)\right].\vspace{-4mm}
\end{equation}

\begin{figure}[bt!]
    \centering
    \begin{subfigure}{0.49\columnwidth}
    \includegraphics[width=\linewidth]{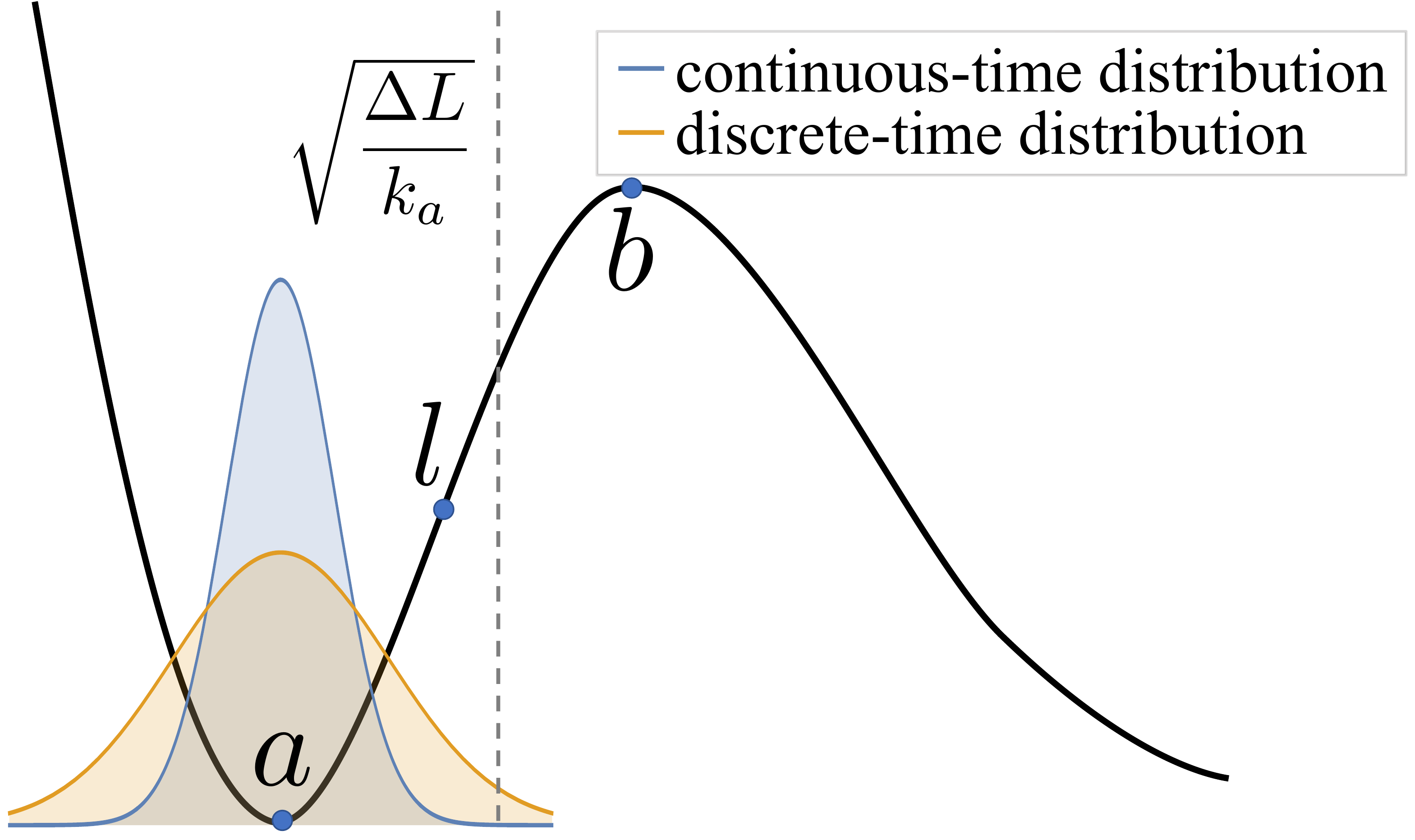}
    \vspace{-1em}
    \caption{Schematic illustration}
    \end{subfigure}
    \begin{subfigure}{0.49\columnwidth}
    \includegraphics[width=\linewidth]{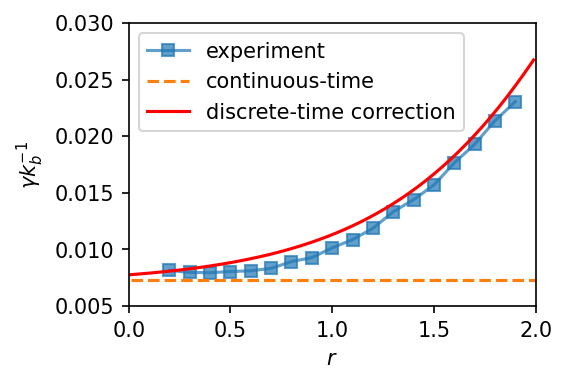}
    \vspace{-2em}
    \caption{Escape rate}
    \end{subfigure}
    \vspace{-0.9em}
    \caption{(a) Schematic illustration of the Kramers problem. The stationary distribution predicted by the discrete-time theory at minimum $a$ is represented by the orange area while the blue area represents the continuous theory prediction. Because the discrete-time distribution spreads out of the valley, we approximate  half of the width of the valley $a$ by $\sqrt{\frac{\Delta L}{k_a}}$ as indicated by the dashed vertical line. (b) $\frac{\gamma}{|k_b|}$ versus $r$. The blue squares are taken from experiments. The orange dashed line shows the prediction of the continuous theory. It agrees with the experiments only for vanishingly small $r$. The red solid curve is our theoretical prediction multiplied by an empirical constant. We see that our prediction is consistent with the experimental data up to a constant coefficient, while the continuous theory underestimates the escape rate. 
    }
    \label{fig:escape}
    \vspace{-1em}
\end{figure}


In Figure~\ref{fig:escape}(b) we plot the quantity $\frac{\gamma}{|k_b|}$ while rescaling the loss function by a factor $r$: $L \to rL$. The continuous theory predicts a constant escape rate by varying $r$, whereas the discrete theory expects a monotonic increase as $r$ becomes large. Such monotonic increase is indeed observed in experiments. The theoretical curve is rescaled by a constant to make comparison easier. Our prediction is qualitatively consistent with the experiment, whereas the continuous theory is only valid in a rather limited range of small $r$. One surprising result here is that, the escape rate of continuous-time dynamics is invariant to the multiplication of the loss function by a constant $r$, while this is not true for discrete-time dynamics, where larger $r$ leads to a larger escape rate from a sharp minimum. 

\vspace{-3mm}
\subsection{Second-Order Methods}\label{sec: second order}
\vspace{-1mm}

In this subsection, we show that the proposed theory can also be extended to analyze the stochastic versions of second-order methods. Here, we consider the stochastic version of the damped Newton's method (DNM) \citep{nesterov2018lectures} where $\Lambda = \lambda K^{-1}$ and the natural gradient descent (NGD) algorithm in which the learning rate matrix is defined as $\Lambda := \lambda J(\mathbf{w})^{-1}$ with $J(\mathbf{w}):= \mathbb{E}[\nabla L (\nabla L)^{\rm T}]$ being the Fisher information. The NGD algorithm is known to be \textit{Fisher-efficient} close to a local minimum \citep{amari1998natural, amari2007methods}, namely, it is the fastest possible method to estimate a given statistical quantity, because the Fisher information is the ``natural" metric in the probability space. The following corollaries give the model fluctuations of DNM and NGD under minibatch noise. A detailed discussion is given in Appendix~\ref{sec:ngd}.
\begin{corollary}
    Let the DNM algorithm be updated with noise covariance being $C=\frac{N-S}{NS}K$. The model fluctuation is
\begin{align}
    \Sigma=\frac{1+\mu}{1-\mu}\frac{\lambda}{2(1+\mu)-\lambda}\frac{N-S}{NS}K^{-1}.
\end{align}
\end{corollary}
\begin{corollary} Let the NGD algorithm be updated with noise covariance being $C=\frac{ N-S}{NS}K\Sigma K$.  Then,
    \begin{equation}
    \Sigma = \lambda \frac{(1+\mu) \frac{N-S}{NS} + 1 - \mu}{2(1-\mu^2)}K^{-1}.
\end{equation}
\end{corollary}
\begin{remark}
For NGD, if $S\to N$, then $C\to 0$. In this situation, we have $\Sigma = \frac{\lambda}{2(1+\mu)} K^{-1}$. This means that the algorithm involves nonzero fluctuations even if the noise is absent! Moreover, the divergence caused by $\lambda k^* \to 2$ disappears for NGD. This suggests that, when the noise is due to minibatch sampling, NGD naturally corrects the ill-conditioned problem of discrete-time SGD.
\end{remark}
\begin{remark}
 Even more interestingly, both the DNM and the NGD algorithms induce fluctuations that are the same as those in the continuous-time SGD algorithm with Gaussian noise up to the coefficients, in the sense that the variance is proportional to $K^{-1}$ which is the local geometry of the minimum. Intuitively, this means that DNM and NGD need no correction even in the discrete-time case; moreover, they are likely to generalize better because they better align with the underlying loss function. 
\end{remark}

It would also be interesting to consider the stationary fluctuation of the adaptive gradient methods such as Adam \cite{journals/corr/KingmaB14_adam}. Adam can be seen as an approximate second-order method with a diagonal preconditioning matrix $\hat{\Lambda} = \lambda /\sqrt{\mathbb{E}[ \text{diag}(\mathbf{g}^2)]}$, where $\mathbf{g}^2$ denotes element-wise square. This can be approximated by a non-diagonal preconditioning matrix $\Lambda = \lambda/ \sqrt{\mathbb{E}[\mathbf{g}\mathbf{g}^{\rm T}]}$. Using $\mathbb{E}[\mathbf{g}\mathbf{g}^{\rm T}]$. Such approximation is in fact not bad; as shown in \citet{zhiyi2021distributional}, the diagonal assumption of $\mathbb{E}[\mathbf{g}\mathbf{g}^{\rm T}]$ leads to surprisingly accurate predictions of the statistical properties of Adam for various modern neural architectures. Under this non-diagonal setting, we can solve the stationary covariance of Adam (without momentum). 
\begin{theorem}\label{thm: Adam} $($Stationary fluctuation of Adam$)$ Let the preconditioning matrix of Adam be $\Lambda = \lambda / \sqrt{\mathbb{E}[\mathbf{g}\mathbf{g}^{\rm T}]}$ and the noise be $C=cK\Sigma K$ with a positive $c$, then\vspace{-2mm}
\begin{equation}
    \Sigma_{\rm Adam}= \frac{\lambda^2 (1+c)}{4}I_D.\vspace{-2mm}
\end{equation}
\end{theorem}
\begin{remark}
Similar to SGD, $\Sigma_{\rm Adam}$ is isotropic. In general, we have $C\sim 1/S$. Two key differences exist: (1) that $\Sigma_{\rm SGD}$ vanishes when $S\gg 1$, while that of Adam converge to a finite constant, like NGD; (2) $\Sigma_{\rm Adam}$ scales as $\lambda^2$, while that of SGD scales as $\lambda$, and so Adam has much smaller fluctuation than SGD.
\end{remark}

\vspace{-3mm}
\section{Concluding Remark}\label{sec: conclusion}
\vspace{-1mm}
In this work, we have analyzed the SGD algorithm in a quadratic potential, which is a good approximation close to any local minimum and a common setting in the recent literature. Compared to the related works, our solution is exact, and relies on fewer assumptions than previous works, and, with the exact solutions, corrections to the previous results that were based on continuous-time approximation are obtained. In fact, we showed that even in the simplest settings, the prediction of the continuous-time solution may deviate significantly and eventually fails for a large learning rate. 
This suggests the fundamental limitation of making the continuous-time approximation in analyzing machine learning problems. 
Previous works have shown that, SGD leads to a flatter minimum due to the existence of anisotropic noise; this anisotropy is enhanced when the dynamics is discrete-time; this gives stronger mobility to model parameters along the sharper directions in the Hessian, and therefore, makes a flatter minimum more favorable. On the other hand, the distortion that the discrete-time SGD causes, in general, does not match the underlying landscape, indicating that using a large learning rate may cause a larger approximation error and worsened generalization. This tradeoff has been discussed only in a restricted setting in this work, and we hope the discussions here may stimulate further research in this direction.

\vspace{-3mm}
 \section*{Acknowledgment}
 \vspace{-1mm}
 We acknowledge Prof. Chikara Furusawa, Zeke Xie and Takashi Mori for valuable discussions.
 This work was supported by KAKENHI Grant No. JP18H01145 and a Grant-in-Aid for Scientific Research on Innovative Areas “Topological  Materials Science (KAKENHI Grant No. JP15H05855) from the Japan Society for the Promotion of Science.


\bibliographystyle{apalike}
\bibliography{ref}

\clearpage
\appendix
\onecolumn


\section{Example of Failure of Continuous-Time Theory}\label{app: example}
\begin{figure}[bt!]
    \centering
    \includegraphics[width=0.5\linewidth]{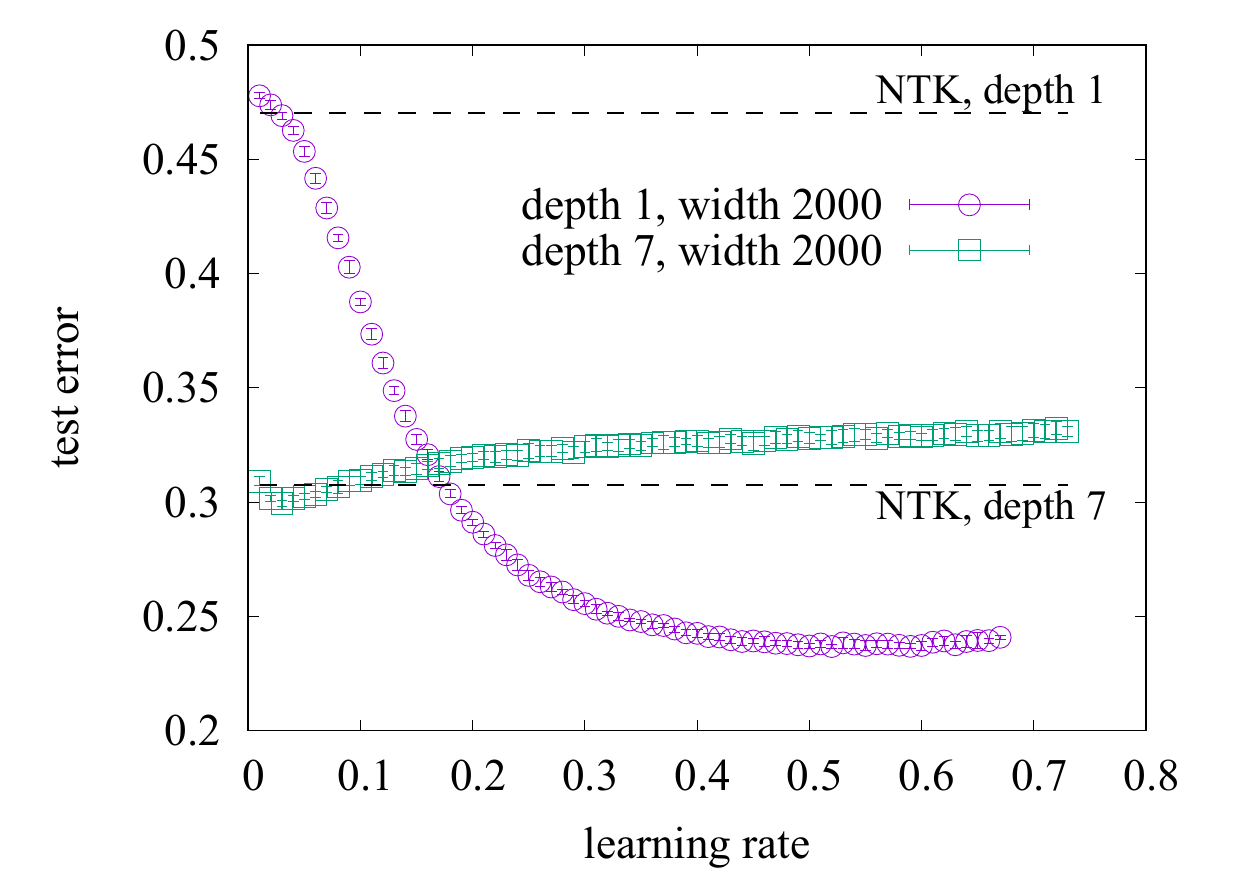}
    \vspace{-.5em}
    \caption{Learning rate dependence of the generalization performance. Nonlinear feedforward neural networks of different depths are trained on a simple task with varying learning rates. We see that, when the learning rate is vanishingly small so that the continuous-time approximation is good, the continuous neural tangent kernel (NTK) provides an accurate characterization of the result of training. However, as the learning rate becomes large, the learning deviates significantly and qualitatively from the NTK prediction,  sometimes for the better, sometimes for the worse. Reproduced from \citet{mori2020deeper}. For other interesting experiments concerning large learning rate, see \citet{Lewkowycz2020}.}
    \label{fig:large_rate_eg app}
\end{figure}
See Figure~\ref{fig:large_rate_eg app} for an example on the generalization performance with different learning rates. For small learning rates, the continuous-time neural tangent kernel (NTK) theory successfully predicts the correct behavior. For a slightly larger $\lambda$, the prediction given by continuous theory deviates significantly from the experiments. 

\vspace{-2mm}
\section{Effect of Overparametrization}\label{app: overparametrization}
\vspace{-1mm}
One particular topic that is of interest in the recent deep learning literature is the role of overparametrization \citep{neyshabur2018the}. Modern neural networks, defying the traditional way of thinking in statistical learning, often perform better when the number of parameters is larger than the number of data points. We comment that our formalism can also be extended straightforwardly to deal with this. In the overparametrized regime, many directions in the loss landscape are degenerate, and have zero curvature; this means that the Hessian matrix in a local minimum is positive semi-definite with many zero eigenvalues. In this situation, the difference between artificially added noise that is usually full-rank and a low-rank noise that is, e.g., proportional to the Hessian becomes important: on the one hand, when the Hessian is low rank, a full-rank noise causes an unconstrained Brownian motion in the null space, the model will thus diverge and one cannot expect to obtain good generalization here; on the other hand, a noise that is proportional to the Hessian only diffuses in the subspace spanned by the Hessian and will not diverge; this is exactly the result obtained in \citet{hodgkinson2020multiplicative} using the formalism of iterated random functions. This implies that the generalization performance induced by minibatch sampling is better than that of an artificially injected Gaussian noise, which has been observed frequently in experiments \citep{Hoffer2017, Zhu2019}.

\clearpage

\section{Additional Experiments for Non-Gaussian Noise}\label{app sec: non gaussian noise}

See Figure~\ref{fig:general noise}. We show that, for example, the theory agrees with the cases when the noise obeys the Student's t-distribution (heavy tail) and the $\chi^2$ distribution (asymmetric); the setting is the same as in the 1d experiments in section~\ref{sec: experiments}. Also, this result remains valid even if $C=C(\Sigma)$ is dependent on the $\Sigma$ itself.
\begin{figure}[hbt!]
    \centering
    \includegraphics[width=0.4\linewidth]{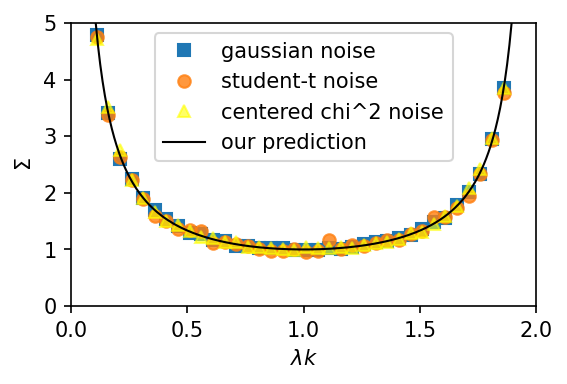}
    \vspace{-1.5em}
    \caption{Comparison of the theoretical prediction with different kinds of noise. We choose the Student's t-distribution with $\nu=4$ as an example of heavy-tail noise with a tail exponent $5$, and a centered $\chi^2$ distribution (by subtracting the mean from a standard $\chi^2$ distribution with degree of freedom $3$). The agreement is excellent, independent of the underlying noise distribution.} 
    \vspace{-1em}
    \label{fig:general noise}
\end{figure}

\section{Additional Experiments for SGD with Momentum}\label{app: SGDM exp}
\begin{figure}[hbt!]
    \centering
    \begin{subfigure}{0.4\linewidth}
    \includegraphics[width=\linewidth]{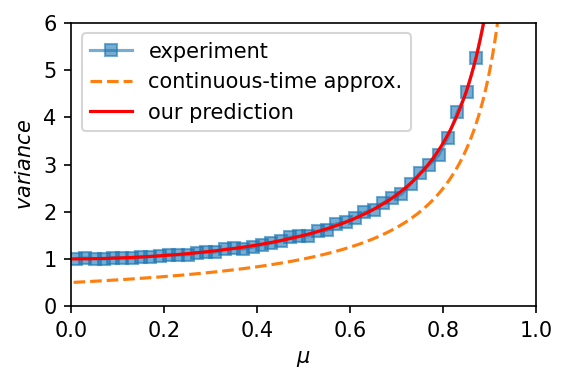}
    \vspace{-2em}
    \caption{White noise}
    \end{subfigure}
    \begin{subfigure}{0.4\linewidth}
    \includegraphics[width=\linewidth]{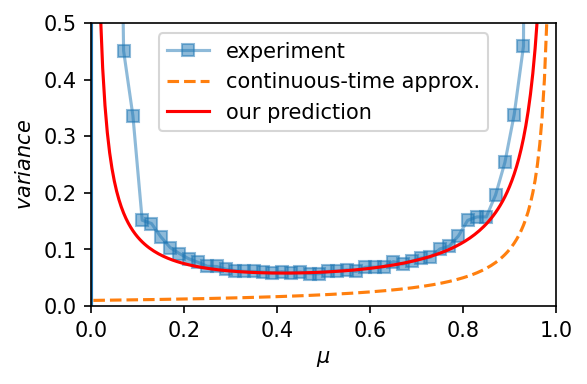}
    \vspace{-2em}
    \caption{Minibatch noise}
    \end{subfigure}
    \caption{Comparison between the continuous-time theory and the discrete-time theory in the presence of momentum. 
      (a) White noise with $\lambda k =1$. In this case, the fluctuation does not diverge when $\mu<1$. However, the error of the continuous-time approximation does not diminish even if $\mu$ gets large.
    (b) Minibatch noise with $\lambda k =2$. Even in the presence of minibatch noise, the proposed theory agrees much better with the experiments. }
    \label{fig:1d momentum appx}
\end{figure}
 In Figure~\ref{fig:1d momentum appx}(a), we plot the model fluctuation with white noise with $\lambda k =1$; this is the case in which there is no divergence for $\mu<1$. Here, we see that the continuous-time theory predicts an error that does not diminish even if $\mu$ is close to $1$. In Figure~\ref{fig:1d momentum appx}(b), we show the experiments with minibatch noise for the same linear regression task adopted in section~\ref{sec: experiments}. The predicted discrete-time result agrees better than the continuous-time one. On the other hand, the agreement becomes worse as the fluctuation in $w$ becomes large. This again suggests the limitation of the commonly used approximation of minibatch noise, i.e., $C\sim H(w)=K$.
 
\clearpage
\section{Proofs and Additional Theoretical Considerations}\label{app: all proofs}
\subsection{Proofs in Section~\ref{sec: theory of SGD}}\label{app: theory of SGD}
\subsubsection{Proof of Theorems~\ref{thm: SGD}, \ref{theo: main theorem with momentum} and \ref{thm: preconditioning matrix eq}}\label{app: derivation of thm momentum}
Because Theorems~\ref{thm: SGD} and \ref{theo: main theorem with momentum} can be derived from Theorem~\ref{thm: preconditioning matrix eq} by assuming a scalar $\lambda$ and $\mu=0$ accordingly, we first prove Theorem~\ref{thm: preconditioning matrix eq}.
\begin{proof}
We assume that the stationary distributions of both $\mathbf{m}$ and $\mathbf{w}$ exist and $\lim_{t\to \infty}\mathbb{E}[\mathbf{w}_t\mathbf{w}_t^{\rm T}] := \Sigma$. The goal is to find $\Sigma$. We assume that $\mathbf{w}_0$ and $\mathbf{m}_0$ are sampled from the stationary distribution. This is valid as long as we are interesting in the asymptotic behavior of $\mathbf{w}_t$. By definition,
\begin{align}
    \Sigma:=\mathbb{E}[\mathbf{w}_t\mathbf{w}_t^{\rm T}] &= 
    \mathbb{E}[(\mathbf{w}_{t-1} - \mu \Lambda \mathbf{m}_{t-1} - \Lambda K\mathbf{w}_{t-1} - \Lambda \eta_{t-1})(\mathbf{w}_{t-1} - \mu \Lambda \mathbf{m}_{t-1} - \Lambda K\mathbf{w}_{t-1} - \Lambda \eta_{t-1})^{\rm T}]\nonumber\\
    &= \mathbb{E}[(I_D-\Lambda K)\mathbf{w}_{t-1}\mathbf{w}_{t-1}^{\rm T} (I_D-K\Lambda )] + \mu^2 \Lambda M \Lambda  + \Lambda C \Lambda  - (A + A^{\rm T}), \label{eq: momentum to be solved}
\end{align}
where $A:=\mu (I_D-\Lambda K )\mathbb{E}\left[\mathbf{w}_{t-1}\mathbf{m}_{t-1}^{\rm T}\right]\Lambda$ and $M:=\mathbb{E}\left[\mathbf{m}_{t-1}\mathbf{m}_{t-1}^{\rm T}\right]$. Notice that $\mathbf{w}_0$ is initialized according to the stationary distribution. Therefore, the distribution does not depend on $t$, namely $\mathbb{E}[\mathbf{w}_t\mathbf{w}_t^{\rm T}] = \mathbb{E}[\mathbf{w}_{t-1}\mathbf{w}_{t-1}^{\rm T}]=\Sigma $. For the covariance matrix of the momentum, similarly,
\begin{align}
    \Lambda M \Lambda &= \mathbb{E}[(\mathbf{w}_{t-1} - \mathbf{w}_{t-2})(\mathbf{w}_{t-1} - \mathbf{w}_{t-2})^{\rm T}] \nonumber\\
    &=  2\Sigma - \mathbb{E}[\mathbf{w}_{t-1}\mathbf{w}_{t-2}^{\rm T}] - \mathbb{E}[\mathbf{w}_{t-2}\mathbf{w}_{t-1}^{\rm T}].
\end{align}
For the final two terms $A + A^{\rm T}$, we have
\begin{align}
    A &= \mu(I_D-\Lambda K )\mathbb{E}\left[  \mathbf{w}_{t-1}\mathbf{m}_{t-1}^{\rm T}\right]\Lambda  \nonumber\\
    &=  \mu (I_D-\Lambda K ) \mathbb{E}[\mathbf{w}_{t-1}(\mathbf{w}_{t-2} - \mathbf{w}_{t-1})^{\rm T}] \nonumber\\
    &= -\mu (I_D-\Lambda K )  \Sigma + \mu (I_D-\Lambda K ) \mathbb{E}[\mathbf{w}_{t-1}\mathbf{w}_{t-2}^{\rm T}],\\
    A^{\rm T}&=-\mu  \Sigma (I_D-K \Lambda )  + \mu \mathbb{E}[\mathbf{w}_{t-2}\mathbf{w}_{t-1}^{\rm T}] (I_D-K \Lambda ). 
\end{align}

Therefore, we are left to solve for $\mathbb{E}[\mathbf{w}_{t-1}\mathbf{w}_{t-2}^{\rm T}]$ and its transpose.  Using the fact that the expectation values are time-independent for the stationary state, we obtain
\begin{align}
    \mathbb{E}[\mathbf{w}_{t-1}\mathbf{w}_{t-2}^{\rm T}] &= \mathbb{E}[\mathbf{w}_{t}\mathbf{w}_{t-1}^{\rm T}] = \mathbb{E}[(\mathbf{w}_{t-1} - \mu \Lambda \mathbf{m}_{t-1} - \Lambda K\mathbf{w}_{t-1} - \Lambda \eta_{t-1})\mathbf{w}_{t-1}^{\rm T}] \nonumber\\
    &= (I_D-\Lambda K )\Sigma  -\mu \Lambda \mathbb{E} [ \mathbf{m}_{t-1} \mathbf{w}_{t-1}^{\rm T}] \nonumber\\
    &= (I_D-\Lambda K )\Sigma + \mu \Sigma - \mu \mathbb{E}[\mathbf{w}_{t-2}\mathbf{w}_{t-1}^{\rm T}],\\
    \mathbb{E}[\mathbf{w}_{t-2}\mathbf{w}_{t-1}^{\rm T}] &= \Sigma  (I_D-K \Lambda )+ \mu \Sigma- \mu \mathbb{E}[\mathbf{w}_{t-1}\mathbf{w}_{t-2}^{\rm T}].
\end{align}
From the above two equations, we have
\begin{align}
    \mathbb{E}[\mathbf{w}_{t-1}\mathbf{w}_{t-2}^{\rm T}] &= \frac{1}{1-\mu^2}\left[(I_D-\Lambda K)\Sigma + \mu \Sigma -\mu \Sigma (I_D - K\Lambda)-\mu^2 \Sigma\right],\\
    \mathbb{E}[\mathbf{w}_{t-2}\mathbf{w}_{t-1}^{\rm T}] &= \frac{1}{1-\mu^2}\left[\Sigma (I_D- K \Lambda)+ \mu \Sigma -\mu (I_D - \Lambda K) \Sigma -\mu^2 \Sigma\right].
\end{align}

Finally, substituting these results back into \eqref{eq: momentum to be solved} yields
\begin{align}
     (1-\mu) (\Lambda K\Sigma + \Sigma K \Lambda) - \frac{1+\mu^2}{1-\mu^2}\Lambda K\Sigma K \Lambda+ \frac{\mu}{1-\mu^2}(\Lambda K\Lambda K\Sigma +\Sigma K\Lambda K\Lambda) = \Lambda C\Lambda. 
\end{align}
\end{proof}

While Theorems~\ref{thm: SGD} and \ref{theo: main theorem with momentum} can be proven via the similar method as above, it is easier to derive them from Theorem~\ref{thm: preconditioning matrix eq}. For Theorem~\ref{theo: main theorem with momentum}, we assume a scalar learning rate $\lambda$.
\begin{proof}
Let $\Lambda = \lambda I_D$. Then from Eq.~\eqref{eq: preconditioning matrix eq}, we have
\begin{align}
    (1-\mu)\lambda (K\Sigma + \Sigma K ) - \frac{1+\mu^2}{1-\mu^2}\lambda^2 K\Sigma K + \frac{\mu}{1-\mu^2}\lambda^2( K^2\Sigma +\Sigma K^2)   = \lambda^2 C. 
    \end{align}
\end{proof}

Theorem~\ref{thm: SGD} can be derived from Theorem~\ref{theo: main theorem with momentum} by setting $\mu=0$.
\begin{proof}
Let $\lambda$ be a scalar and $\mu=0$. Then from Eq.~\eqref{matrixeqmomentum}, we have
\begin{align}
        \Sigma K +  K\Sigma -\lambda K\Sigma K = \lambda C.
    \end{align}  
\end{proof}

\subsubsection{Proof of Corollary~\ref{corollary: momentum commuting}}\label{appx: lemma proof}
We first prove a lemma about commutation relations.
\begin{lemma} \label{lemma_commute}
    $[\Sigma, K]=0$, if and only if $[C,K]=0$.
\end{lemma}
\begin{proof}
\begin{enumerate}
    \item We first prove that if $[\Sigma, K]=0$, then $[C,K]=0$, which is straightforward. Equation.~\eqref{matrixeqmomentum} implies that $C$ is a analytical function of $\Sigma$ and $K$, i.e., $C=C(K,\Sigma)$. The commutator is
    \begin{align}
        [C,K]=[C(K,\Sigma),K].
    \end{align}
    If $[\Sigma, K]=0$, it directly follows that $[C,K]=0$.
    
    \item Now we prove the if $[C,K]=0$, then $[\Sigma, K]=0$, which is not so straightforward. 
We introduce simplified notations: $X:=(1-\mu)I_D$ and $Y:=I_D -\lambda K$. By iteration, we have
\begin{align}
    \mathbf{w}_t &= (X+Y)\mathbf{w}_{t-1}-X\mathbf{w}_{t-2}+\lambda\eta_{t-1}\nonumber\\
    &\dots\nonumber\\
    &=g_{t-1}\mathbf{w}_1 - Xg_{t-2}\mathbf{w}_0 + \lambda\sum_{i=0}^{t-1}g_i \eta_{t-1-i},
\end{align}
where the coefficient matrices $g_i$ satisfy the following recurrence relation
\begin{align}
    g_t=(X+Y)g_{t-1}-Xg_{t-2},\quad \text{for}\quad t\ge 2,
\end{align}
where the initial terms are given by
\begin{align}
    g_0 = I_D, \quad g_1 = X+Y.
\end{align}
It follows from the relation $\lim_{t\to \infty}g_t=0$ that
\begin{align}
  \lim_{t\to \infty} \mathbf{w}_t &=\lim_{t\to \infty}\lambda\sum_{i=0}^{t-1}g_i \eta_{t-1-i}\sim \mathcal{N}\left( 0,\lambda^2\lim_{t\to \infty}\sum_{i=0}^{t-1}g_i C g_i \right):=\mathcal{N}\left( 0,\Sigma \right).
\end{align}
Because every $g_i$ is a function of $K$, $[C,K]=0$ is equivalent with $[\Sigma,K]=0$.

\end{enumerate}
\end{proof}

With this lemma, we prove Corollary~\ref{corollary: momentum commuting}.
\begin{proof}
The matrix equation \eqref{matrixeqmomentum} satisfied by the parameter covariance matrix can be equivalently written in the form containing commutators as
\begin{equation}
    \underbrace{(1-\mu)\lambda K \left(2I_D-\frac{\lambda}{1+ \mu}K\right)\Sigma}_{\text{commuting contribution}}  +
    \underbrace{ (1-\mu)\lambda \left(I_D -\frac{\lambda}{1+\mu}K\right)[\Sigma, K] + \frac{\mu}{1-\mu^2}\lambda^2\left[ K,[K,\Sigma]\right]}_{\text{non-commuting contribution}} = \lambda^2 C,\label{matrixeqmomentum2}
\end{equation}
where the non-commuting contribution is finite if $[C,K]\ne 0$. Otherwise, if $[C,K]=0$, we have $[\Sigma, K]=0$ such that the non-commuting term vanishes and the model fluctuation is
\begin{align}
    \Sigma &= \left[ \frac{\lambda K}{1+\mu}\left( 2I_D - \frac{\lambda K}{1+\mu}\right) \right]^{-1}\frac{\lambda^2 C}{1-\mu^2}\nonumber\\
    &:=[\tilde{\lambda}K(2I_D -\tilde{\lambda}K)]^{-1}\tilde{C}, 
\end{align}
where we introduce the following rescaling:
\begin{align}
    \tilde{\lambda}:=\frac{\lambda}{1+\mu},\quad \tilde{C}:=\frac{1+\mu}{1-\mu}C.
\end{align}
\end{proof}
\begin{remark}
We notice that together with $[C,K]=0$, the form of the matrix equation satisfied by $\Sigma$ is invariant under this rescaling:
\begin{align}
    \tilde{\lambda}(K\Sigma+\Sigma K)-\tilde{\lambda}^2 K\Sigma K= \tilde{\lambda}^2 \tilde{C}.
\end{align}
This suggests that the learning rate can be $1+\mu$ times larger.
\end{remark}

\subsubsection{Proof of Theorem~\ref{thm: typical noise} and Corollaries~\ref{cor: typical noise 1} and \ref{cor: typical noise 2}}\label{app: typical noise}
We first prove Theorem~\ref{thm: typical noise}.
\begin{proof}
According to Theorem~\ref{thm: SGD}, if the algorithm is updated under Gaussian noise with covariance matrix $C$, the stationary distribution of the model parameters $\mathbf{w}$ is $\mathcal{N}(0,\Sigma)$, where $\Sigma$ satisfies Eq.~\eqref{matrixeq}. Due to Lemma~\ref{lemma_commute}, we have $[\Sigma, K]=0$ because $C=\sigma^2 I_D + a K$ commutes with $K$. Referring to Corollary~\ref{corollary: momentum commuting}, the model fluctuation is
\begin{align}
    \Sigma=\lambda(\sigma^2 I_D + a K)[K(2I_D -\lambda K)]^{-1}.
\end{align}
\end{proof}

Corollaries~\ref{cor: typical noise 1} and \ref{cor: typical noise 2} can be easily proven from Theorem~\ref{thm: typical noise}.
\begin{proof}
If $\sigma^2=0$, then $\Sigma=a\lambda (2I_D -\lambda K)^{-1}$. If $a=0$, then $\Sigma=\sigma^2\lambda[K(2I_D -\lambda K)]^{-1}$.
\end{proof}

\subsubsection{Proof of Theorem~\ref{thm: SGDM train error}}\label{app: derivation of thm SGDM train error}
\begin{proof}
In the presence of momentum, we multiply both sides of Eq.~\eqref{matrixeqmomentum2} by $R :=\left(2I_D-\frac{\lambda}{1+ \mu}K\right)^{-1}$ to the left to obtain
\begin{align}
(1-\mu)\lambda K \Sigma  +
    RA_1 + RA_2 = \lambda^2 RC,
\end{align}
where $A_1:=(1-\mu)\lambda \left(I_D -\frac{\lambda}{1+\mu}K\right)[\Sigma, K]$ and $A_2:=\frac{\mu}{1-\mu^2}\lambda^2 \left[ K,[K,\Sigma]\right]$ are terms involving commutators. Taking the trace on both sides yields
\begin{equation}
    (1-\mu)\lambda{\rm Tr} [ K\Sigma] +  {\rm Tr}[R A_1] +  {\rm Tr}[R A_2]   = \lambda^2{\rm Tr} \left[\left(2I_D-\frac{\lambda}{1+ \mu}K\right)^{-1} C \right].
\end{equation}
Because the commuting terms $A_1$ and $A_2$ are anti-symmetric by definition and $R$ is symmetric, the traces ${\rm Tr}[R A_1]$ and ${\rm Tr}[R A_2]$ vanish. Finally, we have
\begin{align}
    L_{\rm train}:=\frac{1}{2}{\rm Tr} [ K\Sigma]=\frac{\lambda}{4(1-\mu)}{\rm Tr} \left[\left(I_D-\frac{\lambda}{2(1+ \mu)}K\right)^{-1} C \right].
\end{align}
\end{proof}

 \clearpage

\subsection{Proofs in Section~\ref{sec: applications}}\label{app: applications}
\subsubsection{Proof of Theorem~\ref{thm: Bays}}\label{app: Bays derivation}
\begin{proof}
The goal of this approximate Bayesian inference task is to find the optimal learning rate which minimizes the KL divergence between the SGD stationary distribution $q(\mathbf{w})$ given in Theorem~\ref{thm: SGD} and the posterior \eqref{posterior}. The KL divergence is
\begin{align}
    D_{\rm KL}(q||f)&=-\mathbb{E}_q (\ln f)+\mathbb{E}_q (\ln q)\nonumber\\
    &=\frac{1}{2}\left[N\text{Tr} [K\Sigma] - \ln|NK| -\ln|\Sigma|-D\right]\nonumber,
\end{align}
where $|\cdot |$ is the determinant and $D$ is the dimension of the parameters $\mathbf{w}$.

Suppose that the noise covariance is $C=\frac{N-S}{NS}K$, which is an approximation of the noise induced by minibatch sampling \cite{Hoffer2017}. According to Theorem~\ref{thm: typical noise}, the covariance of the model is
\begin{align}
    \Sigma = \lambda\frac{N-S}{NS}(2I_D - \lambda K)^{-1}.
\end{align}
Therefore, up to constant terms, the KL divergence is
\begin{align}
    D_{\rm KL}\overset{c}{=}\lambda\frac{N-S}{S}\text{Tr}[(2I_D - \lambda K)^{-1}K]-D\ln \lambda +\ln|2I_D - \lambda K| -D.
\end{align}
Taking the derivative with respect to $\lambda$ yields
\begin{align}
    \frac{\partial}{\partial\lambda}D_{\rm KL}=\frac{N-2S}{S}\text{Tr}[(2I_D - \lambda K)^{-1}K]+\lambda\frac{N-S}{S}\text{Tr}[(2I_D - \lambda K)^{-2}K^2]-\frac{D}{\lambda}.\label{KLderivative}
\end{align}
The optimal $\lambda$ is obtained by solving $\frac{\partial}{\partial\lambda}D_{\rm KL}=0$, namely
\begin{align}
    \frac{N-2S}{S}\text{Tr}[(2I_D - \lambda K)^{-1}K]+\lambda\frac{N-S}{S}\text{Tr}[(2I_D - \lambda K)^{-2}K^2]=\frac{D}{\lambda}.
\end{align}
Equivalently, it can be written into Eq.~\eqref{eq: DKL=0}, because $K$ and $(2I_D - \lambda K)^{-1}$ are simultaneously diagonalizable.
\end{proof}

\subsubsection{Proof of Theorem~\ref{thm: discrete escape efficiency}}\label{app: discrete escape efficiency}
We first derive the discrete-time version of the escaping efficiency \eqref{eq: efficiency} presented in Theorem~\ref{thm: discrete escape efficiency}.
\begin{proof}
Because the initial state is the exact minimum, namely $\mathbf{w}_0=0$, the parameters evolve at time $t$ to
\begin{align}
    \mathbf{w}_t=\lambda\sum_{i=0}^{t-1}(I_D-\lambda K)^i \eta_{t-1-i}.
\end{align}
The loss for such parameters is 
\begin{align}
    L(\mathbf{w}_t)=\frac{\lambda^2}{2}\sum_{i=0}^{t-1}\eta_{t-1-i}^{\rm T}K(I_D-\lambda K)^{2i}\eta_{t-1-i}+{\rm cross-terms},
\end{align}
where the cross-terms involve not-equal-time contributions. The expectation value of the loss at time $t$ is
\begin{align}
    E:=\mathbb{E}[L(\mathbf{w}_t)]&=\frac{\lambda^2}{2}\sum_{i=0}^{t-1}\text{Tr}\left[CK(I_D -\lambda K)^{2i}\right]\nonumber\\
    &=\frac{\lambda}{4} \text{Tr}\left[\left(I_D -\frac{\lambda K}{2}\right)^{-1}\left[I_D - (I_D-\lambda K)^{2t}\right]C\right],
\end{align}
where the cross-terms vanish due to the Gaussian property of the noise and in the second line we use the Neumann series that $\sum_{i=0}^{n}A^i=(I_D-A)^{-1}(I_D-A^{n+1})$.

\end{proof}

\subsubsection{Proof of Corollary~\ref{cor: Ed > Ec}}\label{app: Ed > Ec}

\begin{proof}
As a necessary condition, if each component inside the trace of $E_d$ is greater than that of $E_c$, then the trace itself should be so as well. Specifically, we wish to show that
\begin{align}
    \left(1-\frac{\lambda k}{2}\right)^{-1}\left[1-\left(1-\lambda k\right)^{2t}\right]\ge 1-e^{-2\lambda kt},\quad \forall \ 0<\lambda k<2 \ {\rm and}\ t\ge 0.
\end{align}
Equivalently, we wish to show that
\begin{align}
    \left(1-\frac{\lambda k}{2}\right)e^{-2\lambda kt}\ge \left(1-\lambda k\right)^{2t}-\frac{\lambda k}{2}.
\end{align}
Because $e^{-x}\ge 1-x$ for all $x\ge 0$, we have
\begin{align}
    {\rm lhs}&:=\left(1-\frac{\lambda k}{2}\right)e^{-2\lambda kt}\nonumber\\
    & \ge \left(1-\frac{\lambda k}{2}\right)\left(1-\lambda k\right)^{2t}=\left(1-\lambda k\right)^{2t}-\frac{\lambda k}{2}\left(1-\lambda k\right)^{2t}\nonumber\\
    &\ge \left(1-\lambda k\right)^{2t}-\frac{\lambda k}{2}:={\rm rhs}.
\end{align}
\end{proof}

\subsubsection{Proof of Theorem~\ref{thm: efficiency ratio}}\label{app: efficiency ratio}
\begin{proof}
We first elaborate on the condition about the alignment assumption. As in \citet{Zhu2019}, we denote the maximal eigenvalue and the corresponding eigenvector of $C$ as $c_1$ and $v_1$, respectively. We have $u_1^{\rm T}Cu_1\ge u_1^{\rm T}v_1 c_1 v_1^{\rm T} u_1 =c_1 \langle u_1,v_1\rangle^2$. If the maximal eigenvalues of $C$ and $K$ are aligned in proportion, namely $c_1 /\text{Tr}[C] \ge a_1 k_1/\text{Tr}[K]$, and the angle between their eigenvectors is so small that $\langle u_1,v_1\rangle^2\ge a_2$, then we can conclude that $u_1^{\text{T}}Cu_1\ge ak_1\frac{\text{Tr}[C]}{\text{Tr}[K]}$ with $a:=a_1 a_2$.

We then derive the efficiency ratio \eqref{eq: efficiency ratio}. For a single step, it is the same as the continuous-time one \cite{Zhu2019}. Decomposing $\text{Tr}[KC]$, we have
\begin{align}
    \text{Tr}[KC]=\sum_{i=1}^{D}k_i u_i^{\rm T}Cu_i\ge k_1 u_1^{\rm T}Cu_1\ge a k_1^2 \frac{\text{Tr}[C]}{\text{Tr}[K]}.
\end{align}
For the isotropic equivalence of the noise, we have
\begin{align}
    \text{Tr}[K\bar{C}]=\frac{\text{Tr}[C]}{D}\text{Tr}[K].
\end{align}
Therefore, we obtain
\begin{align}
    \frac{\text{Tr}[KC]}{\text{Tr}[K\bar{C}]}\ge aD\frac{k_1^2}{(\text{Tr}[K])^2}\ge aD\frac{k_1^2}{\left[lk_1+(D-l)D^{-d}k_1\right]^2}\approx aD\frac{1}{\left[l+(D-l)D^{-d}\right]^2}=\mathcal{O}(aD^{2d-1}).
\end{align}
Next, for a long-time, the alignment argument should be slightly modified. While the order of eigenvalues of $K$ is the same as that of $(2I_D-\lambda K)^{-1}$ and they share the same set of eigenvectors, the only thing that should be modified in the argument is that the maximal eigenvalues of $C$ and $(2I_D-\lambda K)^{-1}$ are aligned in proportion such that
\begin{align}
    \frac{c_1 }{\text{Tr}[C]} \ge a_3 \frac{(2-\lambda k_1)^{-1}}{\text{Tr}\left[(2I_D-\lambda K)^{-1}\right]},
\end{align}
where $a_3$ is different from $a_1$ in general. Then the final ratio should contain $a':=a_3 a_2$, instead of $a=a_1 a_2$. The remaining derivation is the same as above.
\end{proof}

\subsubsection{Proof of Theorem~\ref{thm: Kramers}}\label{app: Kramers derivation}
\begin{proof}
First we propose a new approximation on $P(w\in V_a)$. The width of the well $a$ is approximated by $2\sqrt{\frac{\Delta L}{k}}$, where $\Delta L:=L(b)-L(a)$ is the height of the potential barrier and $b$ is the position of the barrier top as shown in Figure~\ref{fig:escape}(a). The probability inside well $a$ is approximated by a finite-range Gaussian integral as
\begin{align}
    P(w\in V_a)&\approx\int_{-\sqrt{\frac{\Delta L}{k}}}^{\sqrt{\frac{\Delta L}{k}}}P(w)dw\nonumber\\
    &=P(a)\sqrt{\frac{2\pi C}{\lambda k (2-\lambda k)}}\textrm{erf}\left(\sqrt{\frac{\lambda(2-\lambda k)\Delta L}{C}}\right),
\end{align}
where $\textrm{erf}(z)$ is the error function. This probability is strictly smaller than $1$, which is consistent with our expectations.

The probability current $J$ can be rewritten as
\begin{align}
    \nabla\left[\exp\left(\frac{L(w)-L(l)}{T}\right)P_c(w)\right]=-J\mathcal{D}^{-1}\exp\left(\frac{L(w)-L(l)}{T}\right),
\end{align}
where $l$ is a midpoint on the most probable escape path between $a$ and $b$ such that $k(\mathbf{w})\approx k_a$ in the path $a\to l$ and $k(\mathbf{w})\approx k_b$ in $l\to b$. In a stationary state, the probability current $J$ is a constant and it can be obtained by integrating both sides of the above equation from $a$ to $b$:
\begin{align}
    \textrm{lhs}=-\exp\left(\frac{L(a)-L(l)}{T_a}\right)P_c(a),
\end{align}
and
\begin{align}
    \textrm{rhs}&=-J\int_{a}^{b}\mathcal{D}^{-1}\exp\left(\frac{L(w)-L(l)}{T}\right)dw\nonumber\\
    &\approx -J\mathcal{D}_{b}^{-1}\int_{-\infty}^{\infty}\exp\left(\frac{L(b)-L(l)+\frac{1}{2}(w-b)^{\rm T}k_b (w-b)}{T_b}\right)dw\nonumber\\
    &=-J\mathcal{D}_{b}^{-1}\exp\left(\frac{L(b)-L(l)}{T_b}\right)\sqrt{\frac{2\pi T_b}{|k_b|}},
\end{align}
where we have approximated the integrand on the right-hand side (rhs) because it is peaked around the point $b$ and $\mathcal{D}_{b}=T_b$. When the noise covariance is $C=\frac{1}{S}k_a$, the two ``temperatures" are given by $T_a=\frac{\lambda}{2S}k_a$ and $T_b=\frac{\lambda}{2S}|k_b|$. 

We propose two corrections to the approximation of the current: (1) we replace the continuous-time distribution $P_c (w)$ by the discrete-time one $P(w)=P(a)\exp\left(-\frac{1}{2}w^{\rm T}\Sigma^{-1}w\right)$; (2) the effective ``temperature" at point $a$ is enlarged because the fluctuation is larger. From the distribution, the ``temperature" should be $T_a=\frac{\lambda}{2S}\frac{k_a}{1-\lambda k_a /2}$. Specifically, the current is now approximated as
\begin{align}
    J\approx P(a)\exp\left(-\frac{1}{2}w^{\rm T}\Sigma^{-1}w\right)\exp\left(\frac{L(a)-L(l)}{T_a}-\frac{L(b)-L(l)}{T_b}\right)\sqrt{\frac{|k_b|}{2\pi T_b}}.
\end{align}
Substituting everything into the definition \eqref{eq: def of Kramers} yields the approximated Kramers rate:
\begin{align}
    \gamma\approx&\frac{1}{2\pi}|k_b|\sqrt{\frac{2}{2-\lambda k_a}}\textup{erf} \left(\sqrt{\frac{S(2-\lambda k_a)\Delta L}{\lambda k_a}}\right)\exp\left[-\frac{2S\Delta L}{\lambda}\left(\frac{l(1-\lambda k_a /2)}{k_a}+\frac{1-l}{|k_b|}\right)\right],
\end{align}
\end{proof}
\begin{remark}
We emphasize that our corrections are not precise because the current is a dynamical quantity. To precisely characterize the Kramers rate, it may be necessary to develop a discrete-time version of the Fokker-Planck equation \eqref{eq: FP}. Hence, our corrections do not guarantee the accuracy of the coefficients in the expressions. 
\end{remark}


\subsubsection{More on Approximation Error}\label{subsec: general momentum}
In this subsection we derive the matrix equations satisfied by the stationary distribution of a class of SGD with a more general form of momentum called Quasi-Hyperbolic Momentum (QHM) \citep{ma2018quasihyperbolic,gitman2019understanding}. The update rule is given by
\begin{align}
    \begin{cases}
        \mathbf{g}_t = K\mathbf{w}_{t-1} + \eta_{t-1};\\
        \mathbf{m}_t = \mu \mathbf{m}_{t-1} + (1-\mu)\mathbf{g}_t;\\
        \mathbf{w}_t = \mathbf{w}_{t-1} - \lambda \left[(1-\nu)g_t +\nu \mathbf{m}_t\right],
    \end{cases} \label{generalmomentSGD}
\end{align}
where the additional parameter $\nu\in [0,1]$ interpolates between the usual SGD \eqref{momentumSGD} without momentum ($\nu=0$) and a normalized version of SGD with momentum \eqref{momentumSGD} ($\nu=1$). The covariance of the model parameters is given in the following theorem.
\begin{theorem}$($Model parameters covariance matrix of QHM$)$ Let the algorithm be updated according to Eqs.~\eqref{generalmomentSGD}. Then the covariance matrix $\Sigma$ of the model parameters satisfies the following set of matrix equations:
\begin{align}\label{eq: QHM matrix equations}
\begin{cases}
    &\mu(1+\mu)(X+X^{\rm T})+\lambda\left[1-\mu\nu(2+\mu)\right](XK+KX^{\rm T})  +a\Sigma +b(K\Sigma+\Sigma K) +cK\Sigma K=d\lambda^2C;\\
    &X=\mu \alpha K^2\Sigma+\lambda[-1+\mu(1+\mu-\mu\nu)]K\Sigma K -\lambda\mu(1-\nu)\alpha K(KQ+QK)K+(1-\mu^2)Q;\\
    &Q-AQA=\Sigma,
    \end{cases}
\end{align}
where
\begin{align}
    &a:=-2\mu(1+\mu),\quad b:=\lambda\mu^2 (1-\nu),\quad c:=\lambda^2\left[1-\mu^2-2\mu\nu(1-\mu)\right],\quad d:=1+\mu[\mu-2\nu-2\mu\nu(1-\nu)],\nonumber\\
    &Q:=\sum_{i=0}^{\infty}A^i \Sigma A^i,\quad A:=\mu[I_D-\lambda(1-\nu)K], \quad \alpha:=\lambda[1-\nu+\nu(1-\mu)].
\end{align}
\end{theorem}
\begin{remark}
By setting $\nu=0$ or $\nu=1$, the previous unnormalized result \eqref{matrixeq} or \eqref{matrixeqmomentum} can be recovered with reparametrization of $\lambda\to\lambda/(1-\mu)$ \citep{gitman2019understanding}. Therefore, this result is the most general one in this work.
\end{remark}

\begin{proof}
The proof of this theorem is essentially similar to that in Appendix~\ref{app: derivation of thm momentum} for SGD with momentum, but more complicated.  By definition, we have
\begin{align}
    \mathbb{E}[\mathbf{w}_t\mathbf{w}_t^{\rm T}]:=\Sigma&=\mathbb{E}\left[(I_D-\alpha K)\mathbf{w}_{t-1}\mathbf{w}_{t-1}^{\rm T}(I_D-\alpha K)\right]+\lambda^2 \nu^2 \mu^2 \mathbb{E}[\mathbf{m}_{t-1}\mathbf{m}_{t-1}^{\rm T}]+\lambda^2 \alpha^2 C \nonumber\\
        &\quad -\lambda \nu \mu\mathbb{E}\left[(I_D-\alpha K)\mathbf{w}_{t-1}\mathbf{m}_{t-1}^{\rm T} + \mathbf{m}_{t-1}\mathbf{w}_{t-1}^{\rm T}(I_D-\alpha K)\right]\nonumber\\
        &=(I_D-\alpha K)\Sigma (I_D-\alpha K)+\lambda^2 \nu^2 \mu^2 M +\lambda^2 \alpha^2 C -(G + G^{\rm T}),\label{eq: QHM Sigma to solve}
\end{align}
where $G:=\lambda \nu \mu (I_D-\alpha K)\mathbb{E}[\mathbf{w}_t\mathbf{m}_t^{\rm T}]$, $M:=\mathbb{E}[\mathbf{m}_t\mathbf{m}_t^{\rm T}]$ and $\alpha:=\lambda[1-\nu+\nu(1-\mu)]$. For momentum, the update rule gives 
\begin{align}
    \lambda \nu \mathbf{m}_t=-\mathbf{w}_t+[I_D-\lambda (1-\nu)K]\mathbf{w}_{t-1}-\lambda (1-\nu)\eta_{t-1}.
\end{align}
Therefore, we have
\begin{align}
    \lambda^2 \nu^2 M=2\Sigma +\lambda^2 (1-\nu)^2 K\Sigma K + \lambda^2 (1-\nu)^2 C- \lambda (1-\nu)(\Sigma K + K \Sigma)-(X+X^{\rm T})+\lambda (1-\nu)(XK+ K X^{\rm T}),
\end{align}
where $X:=\mathbb{E}[\mathbf{w}_t\mathbf{w}_{t-1}^{\rm T}]$. Similarly, this $X$ satisfies
\begin{align}
    X&=(I_D-\alpha K)\Sigma -\lambda \nu \mu \mathbb{E}[\mathbf{m}_{t-1}\mathbf{w}_{t-1}^{\rm T}]\nonumber\\
    &=(I_D-\alpha K)\Sigma+\mu \Sigma - \mu [I_D-\lambda (1-\nu) K]X^{\rm T},\\
    X^{\rm T}&=\Sigma(I_D-\alpha K)+\mu \Sigma- \mu X[I_D-\lambda (1-\nu) K].
\end{align}
The relations between $G$ and $X$ are
\begin{align}
    &G=-\mu (I_D-\alpha K)\Sigma + \mu(I_D-\alpha K)X[I_D-\lambda (1-\nu) K],\\
    &G^{\rm T}=-\mu\Sigma(I_D-\alpha K)+ \mu[I_D-\lambda (1-\nu) K]X^{\rm T}(I_D-\alpha K).
\end{align}
Although no simple expression of $X$ can be obtained, it is possible to provide a set of equations satisfied by $\Sigma$. Substituting everything back into Eq.~\eqref{eq: QHM Sigma to solve} yields a matrix equation involving $\Sigma$ and $X$:
\begin{align}
    \mu(1+\mu)(X+X^{\rm T})+\lambda\left[1-\mu\nu(2+\mu)\right](XK+KX^{\rm T})  +a\Sigma +b(K\Sigma+\Sigma K) +cK\Sigma K=d\lambda^2C,
\end{align}
where $a:=-2\mu(1+\mu)$, $b:=\lambda\mu^2 (1-\nu)$, $ c:=\lambda^2\left[1-\mu^2-2\mu\nu(1-\mu)\right]$, $ d:=1+\mu[\mu-2\nu-2\mu\nu(1-\nu)]$.

Then we try to express $X$ in terms of $\Sigma$. Notice that $X$ and $X^{\rm T}$ satisfy a set of equations with the following form:
\begin{equation}
\begin{cases}
    X+AX^{\rm T}=B,\\
    X^{\rm T}+XA=B^{\rm T},
\end{cases}
\end{equation}
where $A:=\mu[I_D-\lambda(1-\nu)K]$ and $B:=(1+\mu)\Sigma-\alpha K\Sigma$. From them we have
\begin{align}
    X-AXA=B-AB^{\rm T}:=D.
\end{align}
Therefore, by iteration, we have
\begin{align}
    X&=D+AXA=D+A(D+AXA)A=D+ADA+A^2 X A^2=\cdots=\sum_{i=0}^{\infty}A^i D A^i\\
    &=\mu \alpha K^2\Sigma+\lambda[-1+\mu(1+\mu-\mu\nu)]K\Sigma K -\lambda\mu(1-\nu)\alpha K(KQ+QK)K+(1-\mu^2)Q,
\end{align}
where we define $Q:=\sum_{i=0}^{\infty}A^i \Sigma A^i$.

Finally, it can be shown by expanding everything that $Q$ satisfies
\begin{align}
    (I_D-A)Q(I_D+A)+(I_D+A)Q(I_D-A)=2\Sigma.
\end{align}
After simplification, we have
\begin{align}
    Q-AQA=\Sigma.
\end{align}
\end{proof}

From Eqs.~\eqref{eq: QHM matrix equations}, the approximation error for QHM can be calculated.
\begin{corollary}
    The training error for QHM is
    \begin{align}
    L_{\rm train}=\frac{\lambda^2}{2}{\rm Tr} [h(K)^{-1}KC],\label{generalTr}
\end{align}
where
\begin{align}
    &h(K):=\frac{1}{d}\big\{aI_D +2bK+cK^2+[\mu(1+\mu)f(K)+\lambda[1-\mu\nu(2+\mu)] g(K)](I_D -A^2)^{-1}\big\},\label{eq:hK}\\
    &f(K):=2(1-\mu^2)K+\lambda\left[-2+\mu[3+\mu(2-3\nu)\right]K^2,\\
    &g(K):=2(1-\mu^2)I_D +2\lambda[-1+\mu(2+\mu-2\mu\nu)]K-2\lambda\mu(1-\nu)\alpha K^2.
\end{align}
\end{corollary}
\begin{remark}
We emphasize that our result \eqref{generalTr} is exact, whereas the result in \citet{gitman2019understanding} is obtained with a low-order approximation.
\end{remark}
\begin{proof}
By using the similar technique in Appendix~\ref{app: derivation of thm SGDM train error}, Eq.~\eqref{eq: QHM matrix equations} results in
\begin{align}
    h(K)K\Sigma + R =\lambda^2 KC,
\end{align}
where $R$ denotes the terms involving commutative factors such as $[\Sigma,K]$, etc, and
\begin{align}
    &h(K):=\frac{1}{d}\big\{aI_D +2bK+cK^2+[\mu(1+\mu)f(K)+\lambda[1-\mu\nu(2+\mu)] g(K)](I_D -A^2)^{-1}\big\},\\
    &f(K):=2(1-\mu^2)K+\lambda\left[-2+\mu[3+\mu(2-3\nu)\right]K^2,\\
    &g(K):=2(1-\mu^2)I_D +2\lambda[-1+\mu(2+\mu-2\mu\nu)]K-2\lambda\mu(1-\nu)\alpha K^2.
\end{align}
By definition, the approximation error is
\begin{align}
     L_{\rm train}=\frac{1}{2}{\rm Tr} [K\Sigma]=\frac{\lambda^2}{2}{\rm Tr} [h(K)^{-1}KC].
\end{align}
\end{proof}

\subsubsection{Parameter Fluctuations of Second-Order Optimization Methods}\label{sec:ngd}
In this subsection, we deal with the covariance of the stationary distribution obtained by second-order optimization methods.
We first deal with the stationary distribution of Damped Newton's Method (DNM), which is the oldest and most important second-order optimization method, first invented by Newton \citep{nesterov2018lectures}. It is of interest to investigate how the second-order methods behave asymptotically in a stochastic setting.

\begin{theorem}\label{thm: DNM Sigma}$($Model fluctuation of DNM$)$ Let the learning rate matrix be a matrix: $\Lambda:= \lambda K^{-1}$. Then,
\begin{align}\label{eq: DNM Sigma}
    \Sigma=\frac{1+\mu}{1-\mu}\frac{\lambda}{2(1+\mu)-\lambda}K^{-1}CK^{-1}.
\end{align}
\end{theorem}
\begin{proof}
Due to Theorem~\ref{thm: preconditioning matrix eq}, while $\Lambda:= \lambda K^{-1}$, Eq.~\eqref{eq: preconditioning matrix eq} gives
\begin{align}
    \lambda\frac{1-\mu}{1+\mu}\left[2(1+\mu)-\lambda\right]\Sigma=\lambda^2 K^{-1}CK^{-1}.
\end{align}
Therefore, we have
\begin{align}
    \Sigma=\frac{1+\mu}{1-\mu}\frac{\lambda}{2(1+\mu)-\lambda}K^{-1}CK^{-1}.
\end{align}
\end{proof}
\begin{corollary*}
    Suppose that the noise is due to minibatch sampling with the noise covariance being $C=\frac{N-S}{NS}K$. The model fluctuation is
\begin{align}
    \Sigma=\frac{1+\mu}{1-\mu}\frac{\lambda}{2(1+\mu)-\lambda}\frac{N-S}{NS}K^{-1}.
\end{align}
\end{corollary*}
\begin{proof}
Substituting $C=\frac{N-S}{NS}K$ into Eq.~\eqref{eq: DNM Sigma} yields
\begin{align}
    \Sigma=\frac{1+\mu}{1-\mu}\frac{\lambda}{2(1+\mu)-\lambda}\frac{N-S}{NS}K^{-1}.
\end{align}
\end{proof}

From Theorem~\ref{thm: DNM Sigma}, the approximation error for DNM can be calculated.
\begin{corollary}
   The approximation error for DNM is
\begin{equation}
    L_{\rm train}=\begin{cases}
        \frac{1}{2}{\rm Tr}[K\Sigma_{\rm general} ] =  \frac{1+\mu}{1-\mu}\frac{\lambda}{4(1+\mu)-2\lambda} {\rm Tr} [K^{-1}C]; \\
        \frac{1}{2}{\rm Tr} \left[K\Sigma_{\rm minibatch} \right] =\frac{1+\mu}{1-\mu}\frac{D\lambda}{4(1+\mu)-2\lambda}\frac{N-S}{NS}.\\
    \end{cases}
\end{equation}
\end{corollary}
\begin{proof}
The proof is simple by substituting $\Sigma$ into the definition $L_{\rm train}=\frac{1}{2}{\rm Tr} [K\Sigma]$.
\end{proof}

Next, we consider the natural gradient descent (NGD) algorithm. In traditional statistics, the efficiency of any statistical estimator is upper bounded by the Cram\'er-Rao's inequality (CR bound) \citep{rao1992information}. An estimator that achieves the equality in the CR bound is said to be \textit{Fisher-efficient}. A Fisher-efficient method is the fastest possible method to estimate a given statistical quantity. When the gradient descent is used, it is shown \citep{amari1998natural, amari2007methods} that if one defines the learning rate as a matrix, $\Lambda := \lambda J(\mathbf{w})^{-1}$, where $J(\mathbf{w}):= \mathbb{E}[\nabla L (\nabla L)^{\rm T}]$ is the Fisher information, then this optimization algorithm becomes Fisher-efficient in the limit of $t\to \infty$. This algorithm is called the \textit{natural gradient descent} because the Fisher information is the ``natural" metric for measuring the distance in the probability space. The NGD algorithm has therefore attracted great attention both theoretically and empirically \citep{pascanu2013revisiting, amari1998natural}. However, previous literature often takes the continuous-time limit and nothing is known about NGD in the discrete-time regime. We apply our formalism to derive the covariance of the stationary distribution of NGD in the discrete-time regime. To the best of our knowledge, this is the first work to treat the discrete-time NGD and to derive its asymptotic model fluctuations. 

\begin{theorem}\label{thm: equation of NGD}$($Model covaraince matrix of NGD$)$ Let the learning rate  matrix be $\Lambda := \lambda J(\mathbf{w})^{-1}$, where  $J(\mathbf{w})=\mathbb{E}[K\mathbf{w}\mathbf{w}^{\rm T}K]= K\Sigma K$ is the Fisher information. Then the model parameter covariance matrix satisfies the following quadratic matrix equation
\begin{align}\label{eq: matrix eq of NGD}
    (K\Sigma)^2-\frac{\lambda}{2(1+\mu)}K\Sigma-\frac{\lambda}{2(1-\mu)}CK^{-1}=0.
\end{align}
\end{theorem}
\begin{proof}
By setting $\Lambda=\lambda (K\Sigma K)^{-1}$ in Eq.~\eqref{eq: preconditioning matrix eq}, we have
\begin{align}
    2(1-\mu)K^{-1} - \frac{1-\mu}{1+\mu}\lambda K^{-1} \Sigma^{-1} K^{-1}=\lambda K^{-1} \Sigma^{-1} K^{-1} C K^{-1} \Sigma^{-1} K^{-1}.
\end{align}
Multiplying by $K\Sigma K$ to the left and $K\Sigma$ to the right yields
\begin{align}
    (K\Sigma)^2-\frac{\lambda}{2(1+\mu)}K\Sigma-\frac{\lambda}{2(1-\mu)}CK^{-1}=0.
\end{align}
\end{proof}
This matrix equation can be solved while $C$ does not  depend on $\Sigma$ explicitly \citep{Higham2001}.
\begin{corollary}
    Suppose that the noise covariance $C$ is a constant matrix that does not  depend on $\Sigma$ explicitly. Then the solution to Eq.~\eqref{eq: matrix eq of NGD} is
\begin{align}
    \Sigma=\frac{1}{2}K^{-1}\left[Q+\frac{\lambda}{2(1+\mu)}I_D\right], \label{eq: NGD Sigma}
\end{align}
where $Q:=\left[\frac{\lambda^2}{4(1+\mu)^2}I_D + \frac{2\lambda}{1-\mu}CK^{-1}\right]^{\frac{1}{2}}$.
\end{corollary}
\begin{remark}
This result does not seem quite satisfactory, especially because it does not seem to reduce to any meaningful distribution. This means that, when the noise is arbitrary and not related to the use of minibatch sampling, one is not recommended to use NGD\footnote{Recall that the NGD is derived in an online learning setting, where the noise is by definition proportional to the minibatch noise with $N\to\infty$ and minibatch size $1$ \cite{amari1998natural}.}.
\end{remark}
\begin{proof}
By referring to the conclusion in \citet{Higham2001} that the solution to a quadratic matrix equation of the form $AX^2 + BX +C=0$ with $A=I_D$ and $[B,C]=0$ is $X=-\frac{1}{2}B+\frac{1}{2}(B^2-4C)^{1/2}$, Eq.~\eqref{eq: matrix eq of NGD} can be solved explicitly:
\begin{align}
    \Sigma=\frac{1}{2}K^{-1}\left[Q+\frac{\lambda}{2(1+\mu)}I_D\right],
\end{align}
where $Q:=\left[\frac{\lambda^2}{4(1+\mu)^2}I_D + \frac{2\lambda}{1-\mu}CK^{-1}\right]^{\frac{1}{2}}$.
\end{proof}

Now we consider the case where the noise is induced by minibatch sampling. Instead of using the conventional Hessian approximation that $C\approx K$, we here consider a better approximation that $C\approx \frac{N-S}{NS} \mathbb{E}[\nabla L \nabla L^{T}]=\frac{ N-S}{NS}K\Sigma K$. The model fluctuation can be calculated.
\begin{corollary*} Let the NGD algorithm be updated with noise covariance being $C=\frac{ N-S}{NS}K\Sigma K$.  Then,
    \begin{equation}
    \Sigma = \lambda \frac{(1+\mu) \frac{N-S}{NS} + 1 - \mu}{2(1-\mu^2)}K^{-1}.
\end{equation}
\end{corollary*}
\begin{proof}
Substituting $C=\frac{ N-S}{NS}K\Sigma K$ into Eq.~\eqref{eq: matrix eq of NGD} yields
\begin{align}
    (K\Sigma)^2 - \lambda \frac{(1+\mu) \frac{N-S}{NS} + 1 - \mu}{2(1-\mu^2)} K\Sigma =0.
\end{align}
Because $K\Sigma$ is positive definite, we have
    \begin{equation}
    \Sigma = \lambda \frac{(1+\mu) \frac{N-S}{NS} + 1 - \mu}{2(1-\mu^2)}K^{-1}.
\end{equation}
\end{proof}

From Theorem~\ref{thm: equation of NGD}, the approximation error can be calculated.
\begin{corollary}
    The approximation error for NGD is
    \begin{equation}
   L_{\rm train}= \begin{cases}
        \frac{1}{2}{\rm Tr}[K\Sigma_{\rm general} ] =  \frac{1}{4} {\rm Tr} \left[Q+\frac{\lambda}{2(1+\mu)}I_D\right]; \\
        \frac{1}{2}{\rm Tr} \left[K\Sigma_{\rm minibatch} \right] =
        \lambda \frac{(1+\mu) \frac{N-S}{NS} + 1 - \mu}{4(1-\mu^2)} D.\\
    \end{cases}
\end{equation}
\end{corollary}
\begin{proof}
The proof is simple by substituting $\Sigma$ into the definition $L_{\rm train}=\frac{1}{2}{\rm Tr} [K\Sigma]$.
\end{proof}

\subsubsection{Proof of Theorem~\ref{thm: Adam}}\label{app: Adam}
\begin{proof}
Using the non-diagonal approximation, the preconditioning matrix at asymptotic time is
\begin{align}
    \Lambda &= \lambda\mathbb{E}[\mathbf{g}\mathbf{g}^{\rm T}]^{-\frac{1}{2}}\nonumber\\
    &= \lambda\mathbb{E}\left[(K\mathbf{w}+\eta)(K\mathbf{w}+\eta)^{\rm T}\right]^{-\frac{1}{2}}\nonumber\\
    &= \lambda(K\Sigma K+C)^{-\frac{1}{2}}\nonumber\\
    &= \frac{\lambda}{\sqrt{1+c}}(K\Sigma K)^{-\frac{1}{2}}.
\end{align}
Substituting it into Eq.~\eqref{eq: preconditioning matrix eq}, we have
\begin{equation}
    \Lambda K \Sigma + \Sigma K \Lambda - \Lambda K \Sigma K \Lambda = c \Lambda K\Sigma K \Lambda,
\end{equation}
which can be rewritten as
\begin{equation}
    \Lambda^{-1} K^{-1} + K^{-1}\Lambda^{-1} = (1+c)I_D.
\end{equation}
It can be solved to give that
\begin{equation}
    \Sigma = \frac{\lambda^2 (1+c)}{4}I_D.
\end{equation}
\end{proof}
\begin{remark}
The approximation error can be obtained easily as
\begin{equation}
    L_{\rm train}=\frac{1}{2}{\rm Tr}[K\Sigma]=\frac{\lambda^2 (1+c)}{8}{\rm Tr}[K].
\end{equation}
\end{remark}

\end{document}